%% file: main_paper.tex
\documentclass{article}

\usepackage{microtype}
\usepackage{graphicx}
\usepackage{subfigure}
\usepackage{booktabs} 

\usepackage{hyperref}

 \usepackage[preprint]{icml2025}

\usepackage{amsmath}
\usepackage{amssymb}
\usepackage{mathtools}
\usepackage{amsthm}

\usepackage{graphicx}
\usepackage{subfigure}
\usepackage{booktabs} 
\usepackage{makecell}
\usepackage{multirow}
\usepackage{listings}
\usepackage{pythonhighlight}
\usepackage{wrapfig}

\def\method{SeWA}

\usepackage[capitalize,noabbrev]{cleveref}

\usepackage[ruled,algo2e,linesnumbered]{algorithm2e}
\SetKwInput{KwInit}{Init}
\SetCommentSty{mycommentfont}

\theoremstyle{plain}
\newtheorem{theorem}{Theorem}[section]

\newtheorem{lemma}[theorem]{Lemma}

\theoremstyle{definition}
\newtheorem{definition}[theorem]{Definition}
\newtheorem{assumption}[theorem]{Assumption}
\theoremstyle{remark}
\newtheorem{remark}[theorem]{Remark}

\usepackage[textsize=tiny]{todonotes}

\icmltitlerunning{SeWA: Selective Weight Average via Probabilistic Masking}

\begin{document}

\twocolumn[
\icmltitle{\method{}: Selective Weight Average via Probabilistic Masking}



\icmlsetsymbol{equal}{*}

\begin{icmlauthorlist}
\icmlauthor{Peng Wang}{hust}
\icmlauthor{Shengchao Hu}{sjtu}
\icmlauthor{Zerui Tao}{aip}
\icmlauthor{Guoxia Wang}{bidu}\\
\icmlauthor{Dianhai Yu}{bidu}
\icmlauthor{Li Shen}{sysu}
\icmlauthor{Quan Zheng}{hust}
\icmlauthor{Dacheng Tao}{ntu}
\end{icmlauthorlist}

\icmlaffiliation{hust}{School of Mathematics and Statistics, Huazhong University of Science and Technology, Wuhan, China}
\icmlaffiliation{sjtu}{School of Electronic Information and Electrical Engineering, Shanghai Jiao Tong University, Shanghai, China}
\icmlaffiliation{aip}{RIKEN Center for Advanced Intelligence Project, Tokyo, Japan}
\icmlaffiliation{bidu}{Baidu Inc., Beijing, China}
\icmlaffiliation{sysu}{School of Cyber Science \& Technology, Shenzhen Campus of Sun Yat-sen University, Shenzhen, China}
\icmlaffiliation{ntu}{School of Computer and Data Science, Nanyang Technological University, Singapore}

\icmlcorrespondingauthor{Li Shen}{mathshenli@gmail.com}

\icmlkeywords{Machine Learning, ICML}

\vskip 0.3in
]



\printAffiliationsAndNotice{}  

\input{sections/0_Abs}
\input{sections/1_Intro}
\input{sections/2_ReW}
\input{sections/3_Med}

\input{sections/4_Gen_and_Opt}
\input{sections/5_Exp}
\input{sections/6_Con}

 \section*{Impact Statement} The \method{} enhances both the model's generalization ability and convergence speed. Our approach can be further combined with the training process to improve training stability.

\bibliography{example_paper}
\bibliographystyle{icml2025}

\newpage
\appendix
\onecolumn

\input{sections/7_App}


\end{document}

%% file: sections/0_Abs.tex
\begin{abstract}
Weight averaging has become a standard technique for enhancing model performance. However, methods such as Stochastic Weight Averaging (SWA) and Latest Weight Averaging (LAWA) often require manually designed procedures to sample from the training trajectory, and the results depend heavily on hyperparameter tuning. To minimize human effort, this paper proposes a simple yet efficient algorithm called Selective Weight Averaging (\method{}), which adaptively selects checkpoints during the final stages of training for averaging. Based on \method{}, we show that only a few points are needed to achieve better generalization and faster convergence. Theoretically, solving the discrete subset selection problem is inherently challenging. To address this, we transform it into a continuous probabilistic optimization framework and employ the Gumbel-Softmax estimator to learn the non-differentiable mask for each checkpoint. Further, we theoretically derive the \method{}'s stability-based generalization bounds, which are sharper than that of SGD under both convex and non-convex assumptions. Finally, solid extended experiments in various domains, including behavior cloning, image classification, and text classification, further validate the effectiveness of our approach. 

\end{abstract}

%% file: sections/1_Intro.tex
\section{Introduction}
\label{sec:Intro}

Model averaging has demonstrated significant improvements in both practical applications of deep learning and theoretical investigations into its generalization and optimization. From the perspective of generalization, algorithms based on averaging, such as SWA \cite{izmailov2018averaging}, Exponential Moving Average (EMA) \cite{szegedy2016rethinking}, LAWA \citep{kaddour2022stop,sanyal2023early}, and Trainable Weight Averaging (TWA) \citep{li2022trainable}, have been empirically validated to enhance generalization performance. Moreover, most of these methods have been widely adopted, including large-scale network training \citep{izmailov2018averaging,lu2022improving,sanyal2023early}, adversarial learning \citep{xiao2022stability}, etc. In theoretical research, \citet{hardt2016train} and \citet{xiao2022stability} independently give stability-based generalization bounds for the SWA algorithm in different application contexts, and it shows that, under convex assumptions, the generalization bound of the SWA algorithm is half that of SGD. From an optimization perspective, model averaging can help an optimizer approach convergence to quickly reach the optimal solver when it oscillates near a local minimum. In the 1990s, \citet{polyak1992acceleration} demonstrate that averaging model weights improves convergence speed in the setting of convex loss functions. \citet{sanyal2023early} have empirically verified accelerated convergence using the LAWA algorithm in Large Language Models (LLM) pre-training.

Although averaging algorithms exhibit significant advantages theoretically and practically, they heavily rely on manually designed training frameworks and are highly sensitive to selecting hyperparameters. The SWA algorithm, for instance, revisits historical information at every step, leading to slower convergence. Moreover, it requires adopting a cyclic learning rate schedule to locate low-loss points, which introduces additional manual tuning. In contrast, the LAWA algorithm selects the final point for averaging within the last $k$ epochs. \citet{sanyal2023early} have observed from experiments that the final results are not monotonic with respect to the hyperparameter $k$. Instead, they exhibit a pattern where the model's performance initially increases and then decreases as $k$ grows. The TWA algorithm adaptively learns the weights of each model to achieve optimal performance; however, it requires the orthogonalization of two spaces, which adds additional computational costs. In this paper, we aim to propose an averaging algorithm that balances generalization and convergence while minimizing reliance on manually designed training frameworks.

There are two challenges to reaching this goal. 
(1) Ensure the algorithm can achieve both generalization and convergence speeds. Incorporating training information too early may hinder convergence, while collecting a few checkpoints with similar performance at a later stage may lead to limited improvement in generalization.
(2) The adaptive selection of checkpoints can be formulated as a subset selection task, which is a typical discrete optimization problem. Solving such problems requires handling discrete variables that are often non-differentiable.

To address these challenges, we select an interval with size $k$ in the final stage of training and adaptively choose a few checkpoints within this interval for averaging. This approach avoids interference from early-stage information while leveraging the performance benefits of averaging. Furthermore, $k$ is set sufficiently large to ensure that the selected checkpoints can comprehensively explore the solution space. Then, we formulate the SeWA solving process as the corset selection problem, embedding the discrete optimization objective into a probabilistic space to enable continuous relaxation, which makes gradient-based optimization methods available. Furthermore, we address the non-differentiability of binary variables by employing the Gumbel-softmax estimator. In the generalization analysis, we derive generalization bounds for SeWA under different function assumptions based on stability, which are sharper than other algorithms (see \cref{sample-table}). 

\begin{table}[t]
\caption{Comparison of \method{} with other algorithms on different settings.
Here $T$ represents iterations, and $n$ denotes the size of datasets. $L$, $\beta$, and $c$ are constants. $k$ is the number of averaging. $s\in[0,1]$ corresponds to the probability of mask $m=1$ and $\mathcal{O}_s$ means that this upper bound depends on $s$. We can derive that \method{} has sharper bounds compared to others in different settings, where FWA is the general form of LAWA.}
\label{sample-table}
\vspace{-0.2cm}
\begin{center}
\renewcommand\arraystretch{1.0}
\resizebox{0.48\textwidth}{!}{
\begin{sc}
\begin{tabular}{l|c|c}
\toprule
Settings & Algorithm & Generalization Bound  \\ \midrule 
\multirow{4}{*}{convex} 
& SGD & $2\alpha LT/n$ \cite{hardt2016train}\\
& SWA & $\alpha LT/n$ \cite{xiao2022stability}\\
& FWA & $2\alpha L(T-k/2)/n$ \cite{peng2020dfwa}\\
& EMA & $-$\\
& \method{} & $2\alpha Ls(T-k/2)/n$ \cref{thm:stability-conv}\\ \hline
\multirow{4}{*}{\!\!non-convex\!\!} 
&  SGD & $\mathcal{O}(T^{\frac{c\beta}{1+c\beta}}/n)$ \cite{hardt2016train}\\ 
&  SWA & $\mathcal{O}(T^{\frac{c\beta}{2+c\beta}}/n)$ \cite{wanggeneralization}\\ 
&  FWA & $\mathcal{O}(T^{\frac{c\beta}{k+c\beta}}/n)$ \cite{peng2020dfwa}\\ 
& EMA & $-$\\
& \method{} & $\mathcal{O}_s(T^{\frac{c\beta} {k+c\beta}}/n)$\cref{thm:stability-non-with}\\ 
\bottomrule
\end{tabular}
\end{sc}
}
\end{center}
\vskip -0.2in
\end{table}

\subsection{Our Contributions}

This paper primarily introduces the \method{} algorithm. We theoretically demonstrate the \method{} with better generalization performance and give its optimization method. Extensive experiments have been conducted across various domains, including computer vision, natural language processing, and reinforcement learning, confirming the algorithm's generalization and convergence advantages. Our contributions are listed as follows. 

\begin{itemize}
\item Our approach adaptively selects models for averaging in the final training stages, ensuring strong generalization and faster convergence. Notably, \method{} eliminates the need for manually designed selection frameworks, minimizing biases toward specific scenarios.

\item We analyze the impact of masks on generalization theory in expectation and derive a stability-based generalization upper bound, showing advantages over SGD bounds under the different function assumptions.

\item We propose a solvable optimization framework by transforming the discrete problem into a continuous probabilistic space and addressing the non-differentiability of binary variables using the Gumbel-Softmax estimator during optimization.

\item We empirically demonstrate the outstanding performance of our algorithm in multiple domains, including behavior cloning, image classification, and text classification. In particular, the \method{} achieves comparable performance using only a few selected points, matching or exceeding the performance of other methods that require many times more points.
\end{itemize}

%% file: sections/2_ReW.tex
\section{Related Work}
\label{sec:ReW}

\textbf{Weight averaging algorithm.}
Model averaging methods, initially introduced in convex optimization \cite{ruppert1988efficient, polyak1992acceleration,li2023deep}, have been widely used in various areas of deep learning and have shown their advantages in generalization and convergence. Subsequently, with the introduction of SWA \cite{izmailov2018averaging}, which averages the weights along the trajectory of SGD, the model's generalization is significantly improved. Further modifications have been proposed, including the Stochastic Weight Average Density (SWAD) \cite{cha2021swad}, which averages checkpoints more densely, leading to the discovery of flatter minima associated with better generalization. Additionally, Trainable Weight Averaging (TWA) \cite{li2022trainable} has improved the efficiency of SWA by employing trainable averaging coefficients. What's more, other approaches like Exponential Moving Average (EMA) \cite{szegedy2016rethinking} and finite averaging algorithms, such as LAWA \cite{kaddour2022stop, sanyal2023early}, which averages the last $k$ checkpoints from running a moving window at a predetermined interval, employ different strategies to average checkpoints. These techniques have empirically shown faster convergence and better generalization. In meta-learning, Bayesian Model Averaging (BMA) is used to reduce the uncertainty of the model \cite{huang2020meta}. However, these different algorithms often require manual design of averaging strategies and are only applicable to some specific tasks, imposing an additional cost on the training.

\textbf{Stability Analysis.}
Stability analysis is a fundamental theoretical tool for studying the generalization ability of algorithms by examining their stability \citep{devroye1979distribution, bousquet2002stability, mukherjee2006learning, shalev2010learnability}. Based on this, \citet{hardt2016train} use the algorithm stability to derive generalization bounds for SGD, inspiring series works \cite{charles2018stability, zhou2018generalization, yuan2019stagewise, lei2020sharper}. This analysis framework has been extended to various domains, such as online learning \citep{yang2021simple}, adversarial training \citep{xiao2022stability}, decentralized learning \citep{zhu2023stability}, and federated learning \citep{sun2023understanding, sun2023mode}. Although uniform sampling is a standard operation for building stability boundaries, selecting the initial point and sampling without replacement also significantly affects generalization and has been investigated in \citet{shamir2016without, kuzborskij2018data}. For the averaging algorithm, \citet{hardt2016train} and \citet{xiao2022stability} analyze the generalization performance of SWA and establish stability bounds for the algorithm under the setting of convex and sampling with replacement. The main focus of this paper is the generalization and construction of stability bounds for \method{} in both convex and non-convex settings.

\textbf{Mask Learning.}
The general approach involves transforming the discrete optimization problem into a continuous one using probabilistic reparameterization, thereby enabling gradient-based optimization. \citet{coreset} solve the coreset selection problem based on this by using a Policy Gradient Estimator (PGE) for a bilevel optimization objective. \citet{zhangefficient} propose a probabilistic masking method that improves diffusion model efficiency by skipping redundant steps.
While the PGE method may suffer from high variance and unstable training, we solve the mask learning problem using the Gumbel-softmax reparameterization \citep{jang2017categorical,maddison2017the}.
In this paper, we aim to adaptively select checkpoints for averaging to enhance model generalization and convergence, addressing the issue of unstable training.

%% file: sections/3_Med.tex
\section{Methodology}
\label{sec:method}
In  this section, we first give the problem setup and then introduce our proposed \method{} and several terminologies.

\subsection{Problem Setting}
Let $F(w, z)$ be a loss function that measures the loss of the predicted value of the parameter $w$ at a given sample $z$. There is an unknown distribution $\mathcal{D}$ over examples from some space $\mathcal{Z}$, and a sample dataset $S=(z_1, z_2,..., z_n)$ of $n$ examples i.i.d. drawn from $\mathcal{D}$. Then the \emph{population risk} and \emph{empirical risk} are defined as 
\begin{equation}
\textbf{Population Risk:}\quad   \min_w \{R_\mathcal{D}[w]=E_ {z \sim \mathcal{D}}F(w; z) \} \label{prm}
\end{equation}
\begin{equation}
\textbf{Empirical Risk:}\quad  \min_w \{ R_S [w]= \frac{1}{n} \sum_{i=1}^{n}F(w; z_i) \}.\label{erm}
\end{equation}
 The generalization error of a model $w$ is the difference $\epsilon_{gen}=R_\mathcal{D}[w] - R_S [w]$. Moreover, we assume function $F$ satisfies the following \emph{Lipschitz} and \emph{smoothness} assumption.

\begin{assumption}[$L$-Lipschitz]
\label{$L$-Lipschitz}
A differentiable function $F: R^d \rightarrow R$ satisfies the $L$-Lipschitz property, i.e., for $\forall u, v \in R^d, \Vert F(u)-F(v)\Vert \leq L\Vert u-v\Vert$, which implies that the gradient satisfies $\Vert\nabla F(u)\Vert \leq L$.
\end{assumption}

\begin{assumption}[$\beta$-smooth]
\label{beta-smooth}
A differentiable function $F: R^d \rightarrow R$ is $\beta$-smooth, i.e., for $\forall u, v \in R^d$, we have $\Vert\nabla F(u)-\nabla F(v)\Vert \leq \beta\Vert u-v\Vert$.
\end{assumption}

Assumptions \ref{$L$-Lipschitz} and \ref{beta-smooth} are often used to establish stability bounds for algorithms and are crucial conditions for analyzing the model's generalization performance. 

\begin{assumption}[Convex function]
\label{Convex function}
A differentiable function $F: R^d \rightarrow R$ is convex, i.e., for $\forall u, v \in R^d, F(u) \leq F(v) + \langle\nabla F(u), u-v \rangle.$
\end{assumption}

Different function assumptions correspond to different expansion properties, which will be discussed in Lemma \ref{lemma}.

\subsection{SGD and \method{} Algorithm}
For the target function $F$ and the given dataset $S=(z_1, z_2, \cdots, z_n)$, we consider the SGD's general update rule as Eq.~\eqref{SGD-rules}. \method{} algorithm adaptively selects a small number of points for averaging among the last k points of the SGD's training trajectory. It is shown in Eq.~\eqref{\method{}-rules}.

\textbf{SGD} is formulated as
    \begin{equation}\label{SGD-rules}
     w_{t+1} = w_{t} - \alpha \nabla_w F(w_{t},z_{i_t}),
    \end{equation}   
where $\alpha$ is the fixed learning rate, $z_{i_t}$ is the sample chosen in iteration $t$. We choose $z_{i_t}$ from dataset $S$ in a standard way, picking $i_t \sim Uniform\left\{1, \cdots, n \right\}$ at each step. This setting is commonly explored in analyzing the stability \cite{hardt2016train,xiao2022stability}. 

\textbf{\method{}} is formulated as 
\begin{equation}\label{\method{}-rules}
    \bar{w}^{K}_{T}=\frac{1}{K} \sum_{i=T-k+1}^{T} m_{i}w_{i},
\end{equation}
where the mask $m_i \in \left\{0, 1 \right\}$ and $m_i = 1$ indicating the $i$-th weight is selected for averaging and otherwise excluded; the $K = k_{m_i=1}$ corresponds to the number of models selected for averaging. The \method{} algorithm focuses on adaptively selecting a small number of points from the last $k$ steps of the SGD training trajectory for averaging, aiming to improve generalization and accelerate convergence. We show the convergence performance of different
models in Figure \ref{fig:enter-label}.
\begin{figure}
    \centering
    \includegraphics[width=1.\linewidth]{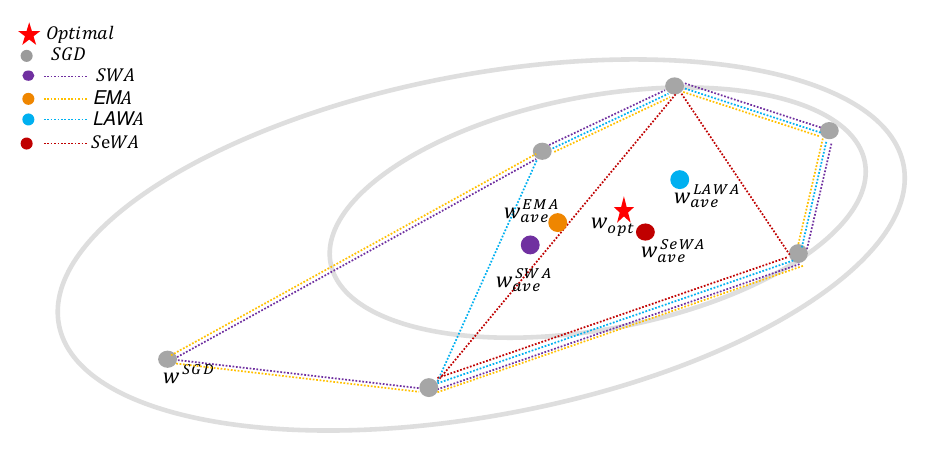}
    \caption{Comparison of \method{} with different models on convergence performance.}
    \label{fig:enter-label}
\end{figure}

\subsubsection{The Expansive Properties}
\begin{lemma}\label{lemma}
Assume that the function $F$ is $\beta$-smooth. Then, \\
{\bf (1). (non-expansive)} If $F$ is convex, for any $\alpha \leq \frac{2}{\beta}$, we have $\Vert w_{T+1}-w_{T+1}^{\prime} \Vert \leq \Vert w_{T}-w_{T}^{\prime}\Vert$; \\
{\bf (2). ($(1\!+\!\alpha\beta)$-expansive)} If $F$ is non-convex, for any $\alpha$, we have $\Vert w_{T+1}\!-\!w_{T+1}^{\prime} \Vert \!\leq\! (1\!+\!\alpha\beta)\Vert w_{T}\!-\!w_{T}^{\prime}\Vert$.
\end{lemma}

Lemma \ref{lemma} tells us that the gradient update becomes $non$-expansive when the function is convex and the step size is small, which implies that the algorithm will always converge to the optimum in this setting. However, although this is not guaranteed when the function is non-convex, it is required that the gradient updates cannot be overly expansive if the algorithm is stable. The proof of Lemma \ref{lemma} is deferred to Appendix \ref{pro-lemma}. Additional relevant results can be found in \citet{hardt2016train,xiao2022stability}. 

\subsection{Stability and Generalization Definition}
\citet{hardt2016train} link the \emph{uniform stability} of the learning algorithm with the expected generalization error bound in research of SGD's generalization. The expected generalization error of a model $w = A_S$ trained by a certain randomized algorithm $A$ is defined as 
    \begin{equation}\label{D-gen}
\mathbb{E}_{S,A}\left[R_{S}\left[A_S\right]-R_\mathcal{D}\left[A_S\right]\right]. 
    \end{equation}
Here, the expectation is taken only over the internal randomness of $A$. 
Next, we introduce the \emph{uniform stability}.

\begin{definition}[$\epsilon$-Uniformly Stable]
A randomized algorithm $A$ is $\epsilon$-uniformly stable if for all data sets $S, S^{\prime} \in Z^{n}$ such that $S$ and $S^{\prime}$ differ in at most one example, we have
    \begin{equation}\label{E-stab}
     \mathop{sup}\limits_{z\in Z}\left\{\mathbb{E}_{A}\left[F(A_S;z)-F(A_{S^{\prime}};z)\right] \right\} \leq \epsilon.
    \end{equation}
\end{definition}

\begin{theorem}{\rm (Generalization in Expectation \citep[Theorem 2.2]{hardt2016train})}
Let $A$ be $\epsilon$-uniformly stable. Then,
    \begin{equation}\label{E-gen}
     \vert\mathbb{E}_{S,A}\left[R_{S}\left[A_S\right]-R_\mathcal{D}\left[A_S\right]\right]\vert \leq \epsilon.
    \end{equation}
\end{theorem}

This theorem clearly states that if an algorithm has uniform stability, then its generalization error is small. In other words, uniform stability implies \emph{generalization in expectation} \cite{hardt2016train}. Above proof is based on \citet[Lemma 7]{bousquet2002stability} and very similar to \citet[Lemma 11]{shalev2010learnability}.

%% file: sections/4_Gen_and_Opt.tex
\section{Generalization Analysis of \method{}}
This section provides two theorems that give upper bounds on generalization in the convex and non-convex settings, respectively. First, a critical lemma is provided for building stability bound in the convex function setting.

\begin{lemma}\label{convex-basic}
Let $\bar{w}_{T}^K$ and $\bar{w}_{T}^{K\prime}$ denote the corresponding outputs of \method{} after running $T$ steps on the datasets $S$ and $S^{\prime}$, which have $n$ samples but only one different. Assume that function $F(\cdot,z)$ satisfies Assumptions \ref{$L$-Lipschitz} and \ref{Convex function} for a fixed example $z\in\mathcal{Z}$, then we have
 \begin{equation}\label{eq_convex-basic}
\mathbb{E}\vert F(\bar{w}_{T}^K;z)-F(\bar{w}_T^{K\prime};z)\vert \leq sL \mathbb{E} [\bar{\delta}_T],
 \end{equation} 
where $s=\sup_{T-k+1\leq i\leq T} s_i$, where $s_i$ denotes the probability of $m_i=1$ and $\bar{\delta}_T=\frac{1}{k}\sum_{i=T-k+1}^{T}\Vert w_i-w_i^{\prime}\Vert$. 
\end{lemma}

\begin{proof} We establish a generalization bound for the algorithm based on stability, where the $L$-Lipschitz transforms the problem into bounding the parameter differences. 
  \begin{equation*}
  \begin{aligned}
  &\mathbb{E}_{z,m,A}\vert F(\bar{w}_{T}^K;z)-F(\bar{w}_T^{K\prime};z) \vert \leq \!L\mathbb{E}_{m,A}\Vert \bar{w}_{T}^K - \bar{w}_T^{K\prime}\Vert\\
  & \leq L \!\left(\!\frac{1}{k}\!\sum_{i=T-k+1}^{T}\!\!s_i \mathbb{E}_{A}\Vert w_i-w_i^{\prime}\Vert+\frac{1}{k}\!\sum_{i=T-k+1}^{T}\!\!(1-s_i)\!\cdot \!0\!\right)\\
  & \leq sL \mathbb{E}_{A} [\bar{\delta}_T], 
  \end{aligned}
 \end{equation*}
where the first inequality comes from $L$-Lipschitz assumption, the second inequality is based on taking the expectation for mask $m_i$, and the last inequality because of $s=\sup_{T-k+1\leq i\leq T} s_i$.
\end{proof}

The Lemma \ref{convex-basic} further decomposes the problem of selecting points for averaging within the last $k$ steps into averaging over the last $k$ steps multiplied by the probability $s_i$ of each step by taking an expectation over the mask. Next, we give the generalization bound for \method{} in the convex setting combined with Lemma \ref{convex-basic}.

\begin{theorem}\label{thm:stability-conv}
Suppose that we run \method{} with constant step sizes $\alpha \leq \frac{2}{\beta}$ for $T$ steps, where each step samples $z \in \mathcal{Z}$ uniformly with replacement. If function $F$ satisfies Assumptions \ref{$L$-Lipschitz}, \ref{beta-smooth} and \ref{Convex function}. \method{} has uniform stability of
\begin{equation}
  \epsilon_{gen} \leq \frac{2\alpha L^2 s}{n} \left(T - \frac{k}{2} \right),
 \end{equation}
where $s=\sup_{T-k+1\leq i\leq T} s_i$.
\end{theorem}

\begin{proof}[Proof sketch]
We can first establish stability bounds for the last $k$ points of the averaging algorithm and then use Lemma \ref{convex-basic} to obtain bounds for \method{}. \\
First, based on Eq.~\eqref{pro-FWA-update}, we divide cumulative gradient into two parts: $\Vert \nabla F(w_{T-1},z_{T}) - \nabla F(w^{\prime}_{T-1},z_{T}) \Vert$ and $\frac{1}{k}\sum_{i=T-k+1}^{T-1} \alpha_i \Vert \nabla F(w_{i-1},z_{i}) - \nabla F(w^{\prime}_{i-1},z_{i}) \Vert$. \\
{\bf(1)} Bound $\Vert \nabla F(w_{T-1},z_{T}) - \nabla F(w^{\prime}_{T-1},z_{T}) \Vert$. On one hand, we consider the different samples $z$ and $z^{\prime}$ is selected with probability $\frac{1}{n}$, we only need to use $L$ to bound $\nabla F(w_{T-1},z_{T})$ and $\nabla F^{\prime}(w^{\prime}_{T-1},z_{T}^{\prime})$ respectively. On the other hand, with probability $1-\frac{1}{n}$ that the same samples $z=z^{\prime}$ is selected, we can use the \emph{non}-expansive update rule from Lemma \ref{lemma}, based on the fact that the objective function is convex and $\alpha\leq\frac{2}{\beta}$. In summary, $\Vert \nabla F(w_{T-1},z_{T}) - \nabla F^{\prime}(w_{T-1}^{\prime},z_{T}) \Vert \leq \frac{2\alpha_T L}{k}$. \\
{\bf(2)} Then we consider bounding the cumulative gradient $\frac{1}{k}\sum_{i=T-k+1}^{T-1} \alpha_i \Vert \nabla F(w_{i-1},z_{i}) - \nabla F(w^{\prime}_{i-1},z_{i}) \Vert$. Since each step $i\in [T-k+1,\cdots, T-1]$ executes sampling with replacement, we can bound them in the way above. Then, we get $\frac{1}{k}\sum_{i=T-k+1}^{T-1}\alpha_i \Vert\nabla F(w^{\prime}_{i-1},z_i) - \nabla F(w_{i-1},z_i) \Vert \leq \frac{2L}{nk}\sum_{i=T-k+1}^{T-1}\alpha_i.$ \\
Second, by merging the above two results and taking summation over $T$ steps, we get $\mathbb{E}\left[\bar{\delta}_{T}\right] \leq (1 \!-\!\frac{1}{n})\bar{\delta}_{T-1}\! +\! \frac{1}{n}\left(\bar{\delta}_{T-1}+\frac{2\alpha_T L}{k}\right)\! +\! \frac{2L}{nk}\sum_{i=T-k+1}^{T-1}\alpha_i \! \leq \! \frac{2\alpha L^2}{n} \left( T - \frac{k}{2} \right)$,
where let $\alpha_i=\alpha$ and substitute it to Eq.~\eqref{eq_convex-basic} yields the desired result.
We leave the proof in Appendix \ref{proof-thm-con-with}.
\end{proof}

\begin{remark}
Theorem \ref{thm:stability-conv} shows that the \method{} algorithm has a sharper bound of $\frac{2\alpha L^2 s}{n} \left(T - \frac{k}{2} \right)$ under the convex assumption than the bound $\frac{2\alpha L^2 T}{n}$ for SGD given by \citet{hardt2016train}. The reason for improving the generalization comes from two main sources: (1) the last $k$ point averaging algorithm improves the SGD bound $\mathcal{O}(T/n)$ to $\mathcal{O}((T-k/2)/n)$, where $k$ is the number of averages, and this result degenerates to the SGD bound when $k = 1$. (2) the \method{} algorithm further improves the bound $\mathcal{O}((T-k/2)/n)$ to $s$ times its size, where $0\leq s \leq 1$ is the probability of $m = 1$.   
\end{remark}

\begin{remark}
The $k$ in Theorem \ref{thm:stability-conv} implies that the more checkpoints involved in the averaging in \method{}, the better the generalization performance. The sparse parameter $s$ selects models, aiming to achieve better generalization with fewer averaged models. However, these two aspects are not contradictory. For a lower sparsity rate, \method{} selects more models, leading to improved generalization performance, a trend validated in our experiments in Section \ref{sec:Exp}.
\end{remark}

\begin{lemma}\label{nonconvex-basic}
Let $\bar{w}_{T}^K$ and $\bar{w}_{T}^{K\prime}$ denote the corresponding outputs of \method{} after running $T$ steps on the datasets $S$ and $S^{\prime}$, which have $n$ samples but only one different. Assume that function $F(\cdot,z)$ satisfies Assumption \ref{$L$-Lipschitz} for a fixed example $z\in\mathcal{Z}$ and every $t_0 \in \{1,\cdots,n\}$, then we have
\begin{equation}\label{eq_nonconvex-basic}
\mathbb{E}\vert F(\bar{w}_{T}^K;z)-F(\bar{w}_{T}^{K\prime};z)\vert \leq \frac{t_0}{n} + sL \mathbb{E}\left[\bar{\delta}_{T}\vert \bar{\delta}_{t_0}=0\right],
\end{equation}   
where $s=\sup_{T-k+1\leq i\leq T} s_i$, where $s_i$ denotes the probability of $m_i=1$ and $\bar{\delta}_T=\frac{1}{k}\sum_{i=T-k+1}^{T}\Vert w_i-w_i^{\prime}\Vert$.
\end{lemma}

\begin{proof} By taking expectation for $m_i$, we split the proof of Lemma \ref{nonconvex-basic} into two parts.\\
In the first part, for $t_0 \in \{1,\cdots,n\}$, we discuss that different samples $z$ and $z^{\prime}$ can be selected only after step $t_0$. Then the inequality $\mathbb{E}\vert F(\bar{w}_{T}^K;z)-F(\bar{w}_{T}^{K\prime};z)\vert \leq \frac{t_0}{n} + L \mathbb{E}\left[\Vert \bar{w}_{T}^K - \bar{w}_{T}^{K\prime}\Vert\vert \Vert \bar{w}_{t_0}^K - \bar{w}_{t_0}^{K\prime}\Vert=0\right]$ will be obtained. One can find further proof details in Appendix \ref{proof-noncon-basic}.\\
Secondly, we take expectation for the $m_i$ of $\Vert \bar{w}_{T}^K \!-\! \bar{w}_{T}^{K\prime}\Vert$, which is similar to the proof of Lemma \ref{convex-basic}.
\end{proof}

\begin{theorem}\label{thm:stability-non-with}
Suppose we run \method{} with constant step sizes $\alpha \leq \frac{c}{T}$ for $T$ steps, where each step samples $z$ from $\mathcal{Z}$ uniformly with replacement. Let function $F\in[0,1]$ satisfies Assumptions \ref{$L$-Lipschitz} and \ref{beta-smooth}. \method{} has uniform stability of
\begin{equation}\label{result-5.3}
  \epsilon_{gen}\leq \frac{1+\frac{1}{c\beta}}{n-1}\left(2csL^2(1+ke^{c\beta})k^{-1}\right)^{\frac{k}{c\beta+k}}\cdot T^{\frac{c\beta}{c\beta+k}},
 \end{equation}
where $s=\sup_{T-k+1\leq i\leq T} s_i$.
\end{theorem}

\begin{proof}[Proof sketch]
In the non-convex setting, we finish the task using Lemma \ref{nonconvex-basic}. First, dividing cumulative gradient in the first stage is the same as in the convex case, except that we use the $(1+\alpha\beta)$-expansive properties (Lemma \ref{lemma}) to bound each $\Vert \nabla F(w_{i-1},z_{i}) - \nabla F(w^{\prime}_{i-1},z_{i}) \Vert$ here. \\
{\bf(1)} We bound each $\Vert \nabla F(w_{i-1},z_{i}) - \nabla F(w^{\prime}_{i-1},z_{i}) \Vert$ in the expectation with probabilities of $1/n$ and $1-1/n$, respectively. And combining $L$-Lipschitz conditions and $(1+\alpha\beta)$-expansive, we have $(2L(1+\alpha\beta)^{i-2})/n$. Then, bounding the cumulative gradient based on the above, we transform it into a problem of summing a finite geometric series. \\
{\bf(2)} We provide a key Lemma \ref{Lemma_noncon}, which helps us to obtain a recurrence relation for $\bar{\delta}_T$ and $\bar{\delta}_{T-1}$ in non-convex. We leave the details in Appendix \ref{proof-Lemma_noncon}. \\
Second, taking summation form $t_0$ to $T$, we get 
\begin{equation*}
 \mathbb{E}\left[\bar{\delta}_{T}\right] \leq \frac{2L(1+ke^{c\beta})}{(n-1)\beta} \cdot \left(\frac{T}{t_0}\right)^{\frac{c\beta}{k}}.   
\end{equation*} \\
Finally, we get $t_0$ by minimizing the Eq.~\eqref{with-con} and plug it into Eq.~\eqref{eq_nonconvex-basic}. We finish the proof and leave the details of this proof in Appendix \ref{proof-thm-non-with}.
\end{proof}

\begin{remark}
Under the non-convex assumption, Theorem \ref{thm:stability-non-with} shows that \method{} has bound $\mathcal{O}(T^{c\beta/(c\beta+k)}/n)$ compared to the bound $\mathcal{O}(T^{c\beta/(c\beta+1)}/n)$ for SGD in \citet{hardt2016train}, again showing its ability to improve generalization significantly. Although the number $k$, closely related to the number of iterations $T$, seems to dominate the result, the direct influence of parameter $s$ on the entire stationary bound also plays a crucial role. 
\end{remark}

\begin{remark}
The assumption that $F(w;z) \in [0,1]$ in Theorem \ref{thm:stability-non-with} is adopted for simplicity. Removing this condition does not affect the final results, as it merely introduces a constant scaling factor. The same setting is commonly used and discussed in \citet{hardt2016train,xiao2022stability}. 
\end{remark}

\section{Practical \method{} Implementation}
Although the \method{} algorithm has simpler expressions, the difficulty is learning the mask $m_i$. Inspired by tasks such as coreset selection \cite{coreset}, the discrete problem is relaxed to a continuous one. We first formulate weight selection into the following discrete optimization paradigm:
\begin{equation}\label{dis-opt}
    \min_{m\in C}F(m)=L\left(\textbf{w}(m)\right)=\frac{1}{n}\sum_{i=1}^{n}l\left(f(x_i;\textbf{w}(m),y_i) \right),
\end{equation}
where $C=\left\{\textbf{m}: m_i = 0 \,\textbf{or}\, 1, \Vert \textbf{m}\Vert_0\leq K \right\}$ and $\textbf{w}(m)=\frac{1}{K} \sum_{i=T-k+1}^{T} m_{i}w_{i}$.

To transform the discrete Eq.~\eqref{dis-opt} into a continuous one, we treat each mask $m_i$ as an independent binary random variable and reparameterize it as a Bernoulli random variable, $m_i \sim \text{Bern}(s_i)$, where $s_i \in [0, 1]$ represents the probability of $m_i$ taking the value 1, while $1-s_i$ corresponds to the probability of $m_i$ being 0. Consequently, the joint probability distribution of $m$ is expressed as $p(m\vert s) = \prod _{i=1}^{n}(s_i)^{m_i}(1 - s_i)^{1 - m_i}$. Then, the feasible domain of the target Eq.~\eqref{dis-opt} approximately becomes $\hat{C}=\left\{s: 0\leq s\leq 1, \Vert s\Vert_1\leq K \right\}$ since $\mathbb{E}_{m_i \sim p(m\vert s)}\Vert m\Vert_0 = \sum_{i=1}^{n}si$. As in the previous definition, $K>0$ in $\hat{C}$ is a constant that controls the size of the feasible domain. Then, Eq.~\eqref{dis-opt} can be naturally relaxed into the following excepted loss minimization problem:
\begin{equation}\label{E-dis-opt}
    \min_{s\in \hat{C}}F(s)=\mathbf{E}_{p(m|s)}L\left(\textbf{w}(m)\right),
\end{equation}
where $\hat{C}=\left\{s: 0\leq s\leq 1, \Vert s\Vert_1\leq K \right\}$.

Optimizing Eq.~\eqref{E-dis-opt} involves discrete random variables, which are non-differentiable.
One choice is using Policy Gradient Estimators (PGE) such as 
the REINFORCE algorithm \citep{williams1992simple,sutton1999policy} to bypass the back-propagation of discrete masks $m$,
\begin{equation*}
    \nabla_s F(s)=\mathbf{E}_{p(m|s)} L\left(\textbf{w}(m)\right) \nabla_s \log p( m \mid s).
\end{equation*}
However, these algorithms suffer from the high variance of computing the expectation of the objective function, hence may lead to slow convergence or sub-optimal results.

To address these issues, we resort to the reparameterization trick using Gumbel-softmax sampling \citep{jang2017categorical,maddison2017the}.
Instead of sampling discrete masks $m$, we get continuous relaxations by,
\begin{equation}\label{eq:gs-sample}
\small \!\!\!\!\!  \tilde{m}_i = \frac{\exp((\log s_i + g_{i, 1}) / t)}{\exp((\log s_i + g_{i, 1}) / t) + \exp((\log (1 - s_i) + g_{i, 0}) / t)},
\end{equation}
for $i = 1, \dots, k$, where $g_{i, 0}$ and $g_{i, 1}$ are i.i.d. samples from the $\text{Gumbel}(0, 1)$ distribution. The hyper-parameter $t > 0$ controls the sharpness of this approximation. When it reaches zero, i.e., $t \to 0$, $\tilde{m}$ converges to the true binary mask $m$. During training, we maintain $t > 0$ to make sure the function is continuous. For inference, we can sample from the Bernoulli distribution with probability $s$ to get sparse binary masks. In practice, the random variables $g \sim \text{Gumbel}(0, 1)$ can be efficiently sampled from,
\begin{equation*}
    g = - \log ( - \log (u)), \quad u \sim \text{Uniform}(0, 1).
\end{equation*}
For simplicity, we denote the Gumbel-softmax sampling in Eq.~\eqref{eq:gs-sample} as $\tilde{m} = \text{GS}(s, u, t)$, where $u \sim \text{Uniform}(0, 1)$. Replacing the binary mask $m$ in Eq.~\eqref{E-dis-opt} with the continuous relaxation $\tilde{m}$, the optimization problem becomes,
\begin{equation*}
    \min_{s\in \hat{C}}F(s)=\mathbf{E}_{u \sim \text{Uniform(0, 1)}} L\left(\textbf{w}(\text{GS}(s, u, t)\right),
\end{equation*}
where $\hat{C}=\left\{s: 0\leq s\leq 1, \Vert s\Vert_1\leq K \right\}$. The expectation can be approximated by Monte Carlo samples, i.e.,
\begin{equation*}
    \min_{s\in \hat{C}} \hat{F}(s)= \frac{1}{M} \sum_{m=1}^M L\left(\textbf{w}(\text{GS}(s, u^{(m)}, t)\right),
\end{equation*}
where $u^{(m)}$ are i.i.d. samples drawn from $\text{Uniform}(0, 1)$.
Empirically, since the distribution of $u$ is fixed, this Monte Carlo approximation exhibits low variance and stable training \cite{kingma2013auto,rezende2014stochastic}. Furthermore, since Eq.~\eqref{eq:gs-sample} is continuous, we can optimize it using back-propagation and gradient methods.

\begin{algorithm2e}[t]
    \caption{Selected Weight Average (\method{})}\label{alg:gs}
    \KwIn{Checkpoints $\textbf{w}$, hyper-parameter $t$}
    \KwInit{Mask probability $s$\;}
    \For{$i = 1, \dots, max\_iteration$}{
    \tcc{Gumbel-softmax sampling}
    \For{$m = 1, \dots, M$}{
        Sample $u^{(m)} \sim \text{Uniform}(0, 1)$\;
        Compute $L\left(\textbf{w}(\text{GS}(s, u^{(m)}, t)\right)$\;
    }
    \tcc{Learning mask probability}
    Optimize $\hat{F}(s)= \frac{1}{M} \sum_{m=1}^M L\left(\textbf{w}(\text{GS}(s, u^{(m)}, t)\right)$\;
    }
    \KwOut{Mask $m$ based on $K$ largest probabilities in $s$}
\end{algorithm2e}

\begin{remark}
\method{} adaptively selects useful checkpoints, which implies that it does not require the extra cost associated with manual design and avoids model biases introduced by prior knowledge, thereby making our approach applicable to a broader range of tasks. In the following experiments, \method{} algorithm demonstrates particular suitability for scenarios characterized by unstable training trajectories, such as behavior cloning. By leveraging checkpoint averaging, \method{} effectively stabilizes the training process, mitigating fluctuations and enhancing overall performance.
\end{remark}

%% file: sections/5_Exp.tex
\section{Experiment}
\label{sec:Exp}
In our experimental evaluation, we systematically explore the performance of our proposed method across three distinct settings: behavior cloning, image classification, and text classification. These settings are chosen to demonstrate the generality and effectiveness of our approach in diverse application domains. 
Details of the experimental setup, including network architectures, hyperparameters, and additional results, are provided in Appendix \ref{sec:ExpDetail}.

\subsection{Behavior Cloning}
\textbf{Experimental Setups.} We performed extensive evaluations using the widely recognized D4RL benchmark \citep{fu2020d4rl}, with a particular focus on the Gym-MuJoCo locomotion tasks. These tasks are commonly regarded as standard benchmarks due to their simplicity and well-structured nature. They are characterized by datasets containing a significant proportion of near-optimal trajectories and smooth reward functions, making them suitable for evaluating the performance of reinforcement learning methods. 
%
For evaluation, we employed cumulative reward as the primary metric, as it effectively captures the overall performance of the agent in maximizing returns over the course of its trajectories. 

\textbf{Baselines.} 
To assess the effectiveness of our proposed SeWA method, we compare it against several established baselines, including the original pre-training recipe based on stochastic gradient descent (SGD), Stochastic Weight Averaging (SWA) \citep{izmailov2018averaging}, and Exponential Moving Average (EMA) \citep{szegedy2016rethinking}, which we adapt for the behavior cloning setting.
For EMA, we follow the approach outlined in \citet{kaddour2022stop}, setting the decay parameter to 0.9 and updating the EMA model at every $K$ training step, which is a widely adopted standard practice. 
For SWA, we adhere to the original pre-training procedure up to 75\% completion. Following this phase, we initiate SWA training with a cosine annealing scheduler and compute the SWA uniform average every $K$ steps to aggregate model parameters effectively.
Additionally, we compare our SeWA with LAWA \citep{sanyal2023early} and a Random baseline. 
Both of these baselines involve directly averaging pretrained checkpoints from the original pre-training process without additional retraining. Specifically, LAWA selects $K$ checkpoints at equal intervals, whereas the Random baseline selects $K$  checkpoints randomly from the same set.
It is important to note that LAWA, Random, and our proposed method all utilize the final 1000 checkpoints from the pre-training process to compute performance metrics without further retraining the model.
In contrast, the SGD, SWA, and EMA baselines report their final performance directly, as their evaluation pipelines and corresponding techniques are integrated into their respective training processes. This ensures a fair and consistent comparison across all methods.

\begin{figure}
    \centering
    \includegraphics[width=0.9\linewidth]{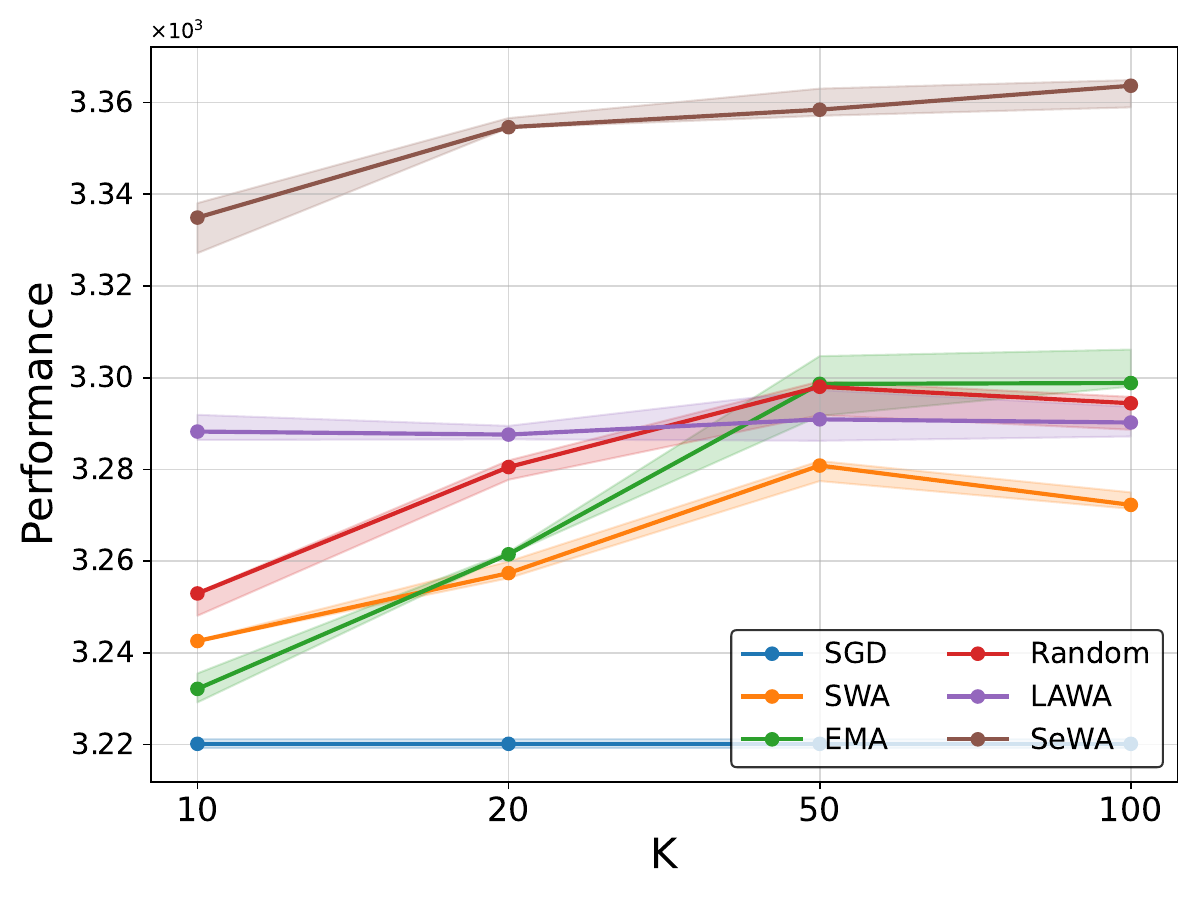}
    \vspace{-.4cm}
    \caption{Comparison of different methods on the D4RL benchmark. Each data point represents the average cumulative reward across multiple tasks, averaged over 3 random seeds and 20 trajectories per seed. Detailed results are provided in Appendix \ref{sec:ExpDetail}.}
    \label{fig:BC}
\end{figure}

\begin{figure*}[t!]
    \centering
    \subfigure{
    \centering
    \includegraphics[width=0.3\textwidth]{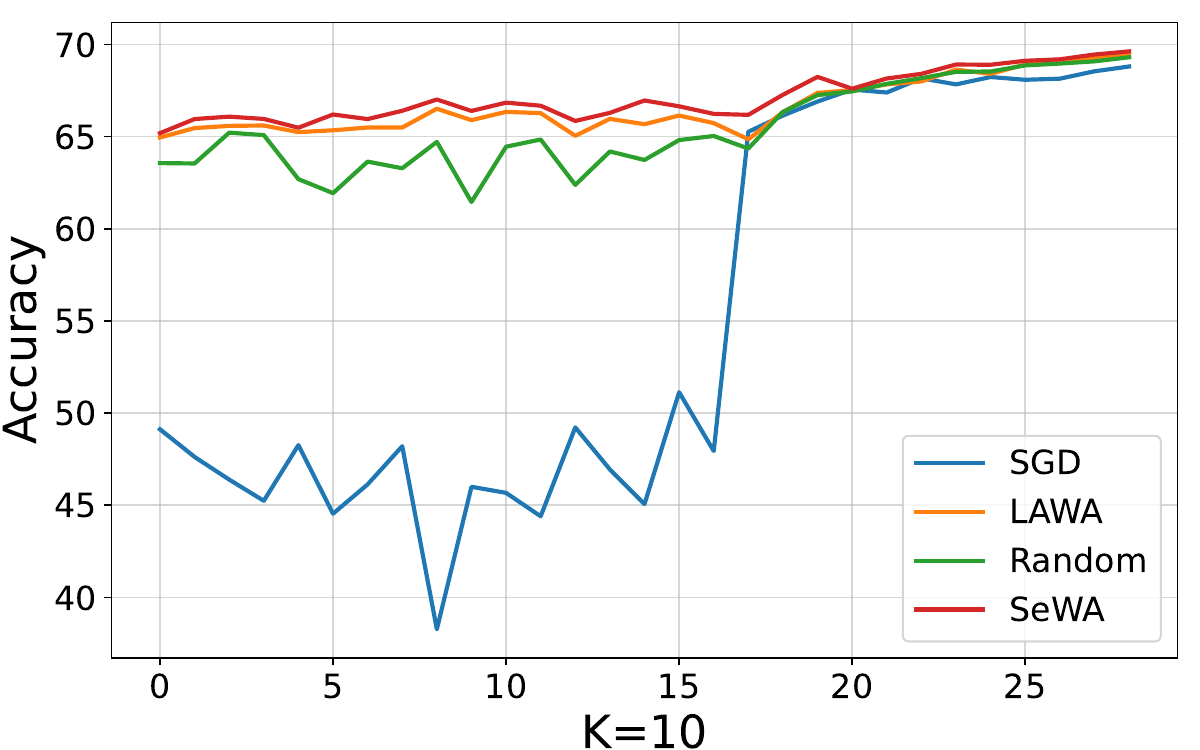}}
    \centering
    \subfigure{
    \centering
    \includegraphics[width=0.3\textwidth]{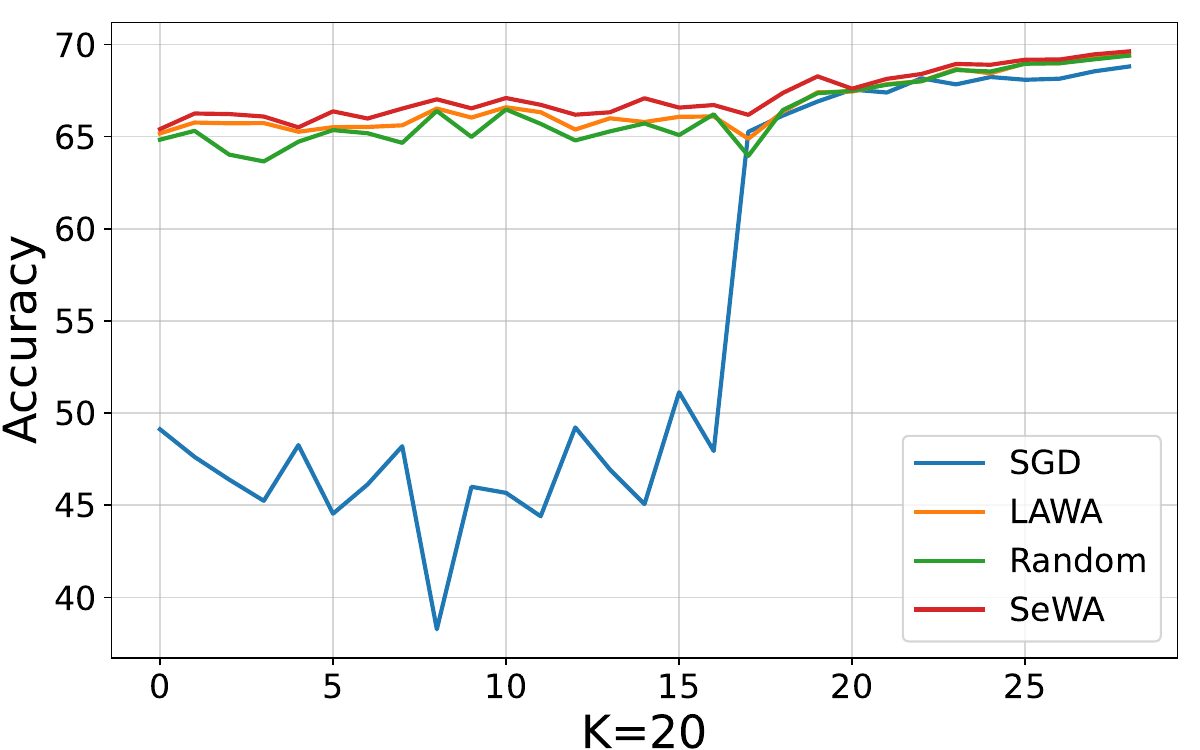}}
    \subfigure{
    \centering
    \includegraphics[width=0.3\textwidth]{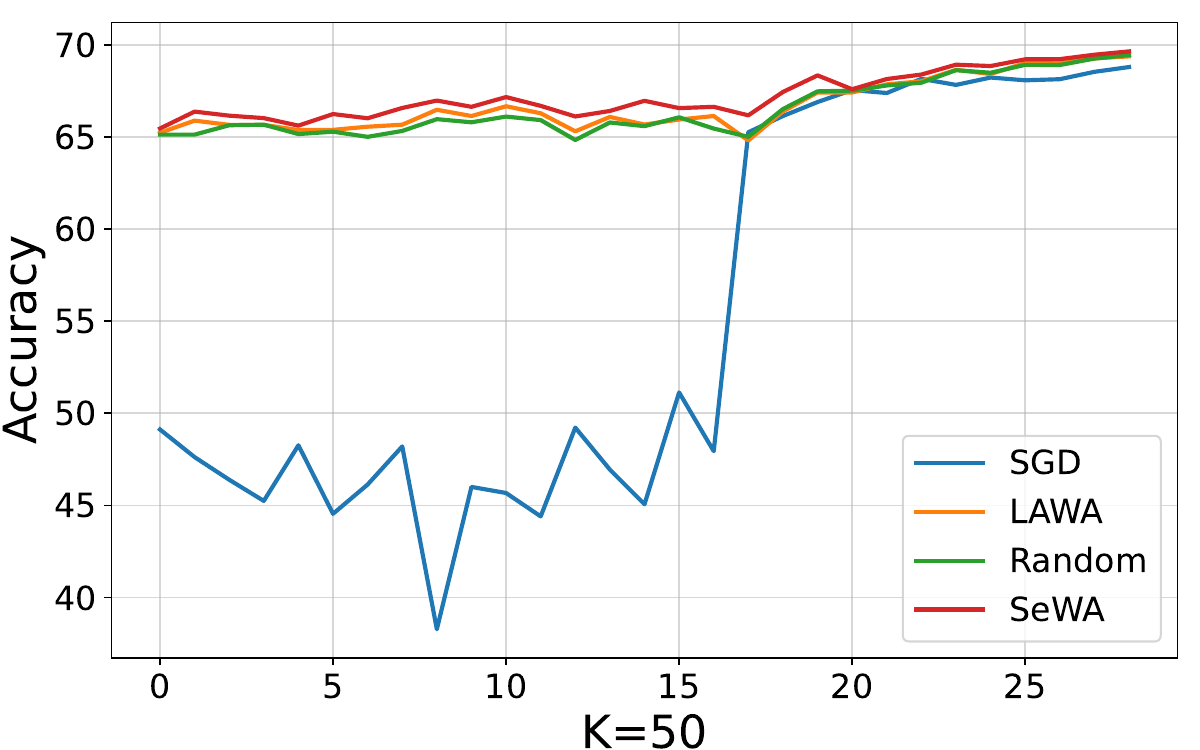}}
    \vspace{-0.2cm}
    \caption{From left to right, the figures illustrate the impact of the hyperparameter $K$ on the CIFAR-100 task. Each point corresponds to intervals of 100 checkpoints, with $K$ checkpoints selected and averaged from these intervals using different strategies.
    }
    \label{fig:cifar100}
\end{figure*}

\begin{figure*}[t!]
    \centering
    \subfigure{
    \centering
    \includegraphics[width=0.3\textwidth]{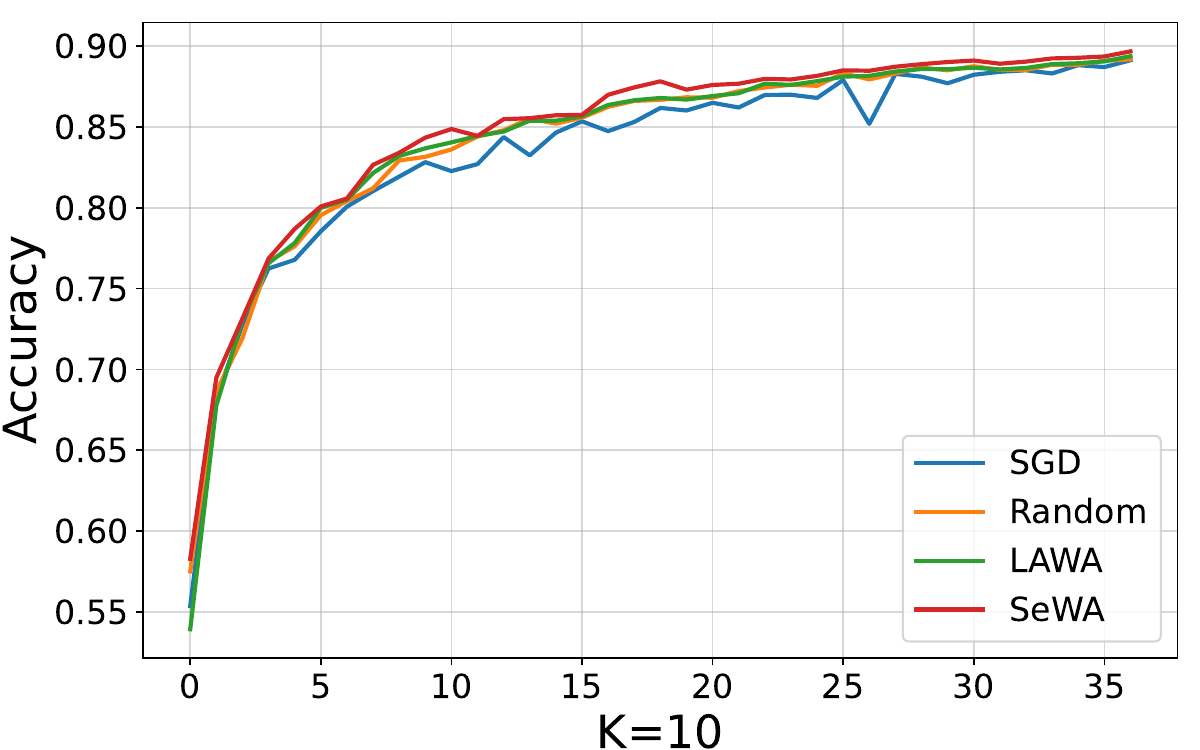}}
    \centering
    \subfigure{
    \centering
    \includegraphics[width=0.3\textwidth]{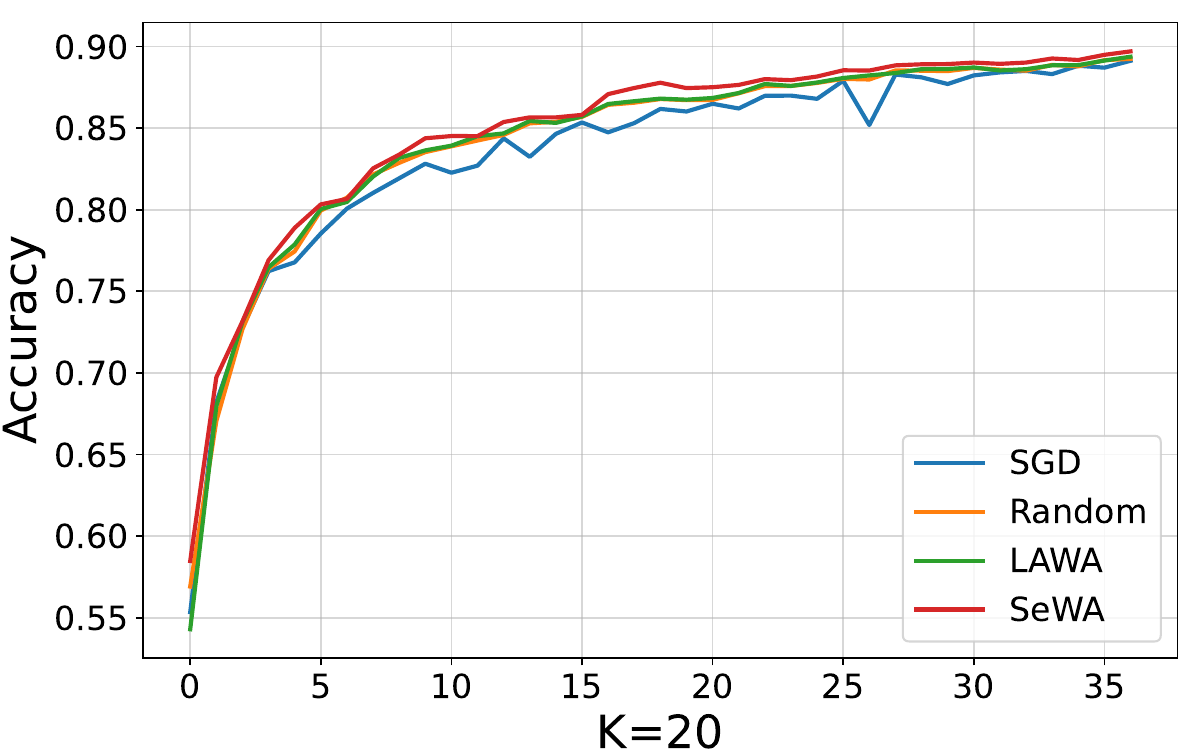}}
    \subfigure{
    \centering
    \includegraphics[width=0.3\textwidth]{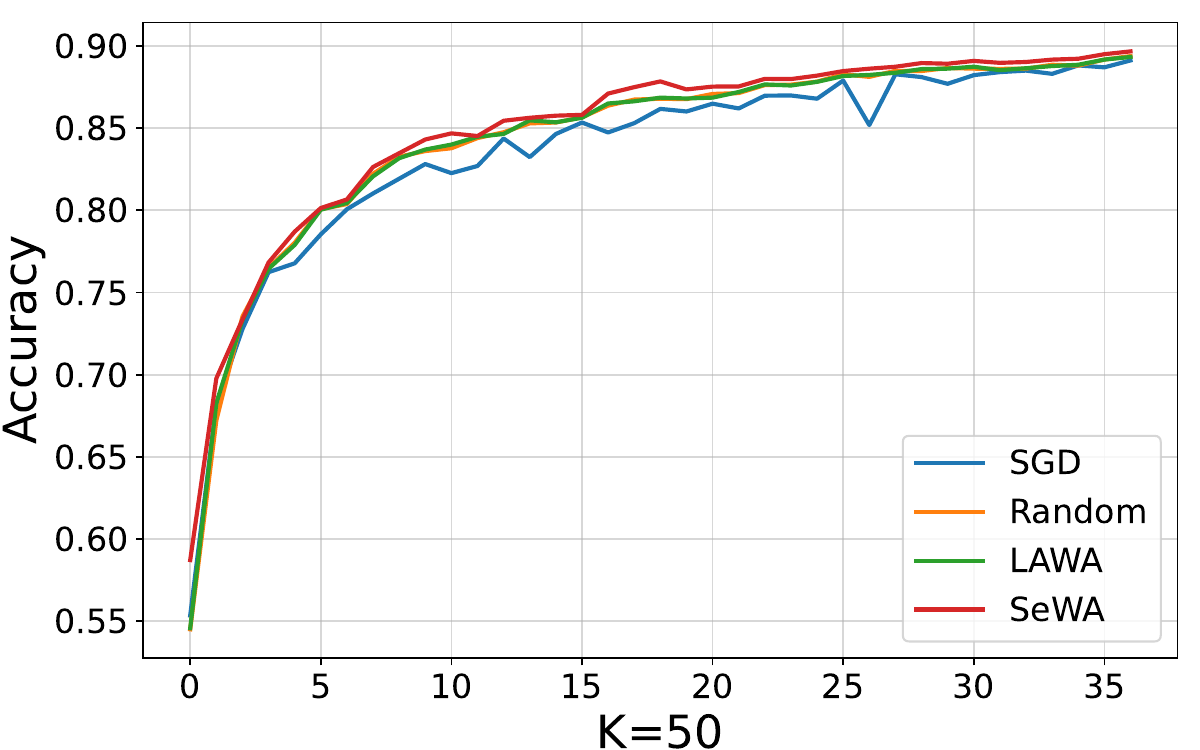}}
    \vspace{-0.2cm}
    \caption{From left to right, the figures illustrate the impact of the hyperparameter $K$ on the AG News corpus. Each point corresponds to intervals of 100 checkpoints, with $K$ checkpoints selected and averaged from these intervals using different strategies.
    }
    \label{fig:ag}
\end{figure*}

\textbf{Results.} 
As shown in Figure \ref{fig:BC}, all baselines demonstrate superior performance compared to the original SGD optimizer, highlighting the effectiveness of weight averaging strategies in improving model performance. These results confirm that weight averaging can serve as a valuable technique for stabilizing and enhancing model training outcomes. 
Additionally, our analysis reveals that increasing the number of checkpoints $K$ used for averaging consistently improves performance across all methods. However, this improvement tends to plateau beyond a certain threshold, indicating diminishing returns as the number of averaged checkpoints increases.
Most significantly, our proposed method consistently outperforms all other baselines across all experimental settings. Remarkably, even with only $K=10$ checkpoints used for averaging, our method achieves superior results compared to competing approaches that utilize $K=100$ checkpoints. This highlights our approach's efficiency and robustness, as it can deliver significant improvements with a substantially smaller computational footprint. These results demonstrate the scalability and practicality of our method in scenarios where resource efficiency is critical.

\subsection{Image Classification}
\textbf{Experimental Setups.} 
To evaluate our method's performance in image classification, we utilize the CIFAR-100 dataset and the ResNet architecture. With its diverse set of 100 classes, the CIFAR-100 dataset provides a challenging benchmark for image classification tasks. We use classification accuracy on the test dataset as the primary performance metric.
In our experiments, we utilize intermediate model checkpoints saved during the final stage of training, specifically after 10,000 training steps. 
Performance is evaluated at intervals of 100 checkpoints, with the number of checkpoints included in the averaging procedure within each interval controlled by the hyperparameter $K$.
This flexibility allows us to adjust the extent of checkpoint aggregation and analyze its impact comprehensively.
By evaluating our method under varying levels of checkpoint averaging, derived from different intervals of the training process, this setup facilitates a robust assessment of its effectiveness across diverse configurations. 


\textbf{Results.}
As illustrated in Figure \ref{fig:cifar100}, all baselines outperform the original SGD optimizer, underscoring the effectiveness of weight averaging in enhancing model performance. Additionally, weight averaging accelerates model convergence, with all baselines reaching performance levels that SGD requires $17$ steps to achieve.
Our SeWA method consistently delivers the best performance, demonstrating its effectiveness. Beyond $17$ steps, where the model approaches convergence, further improvement becomes minimal, as the checkpoints at this stage share highly similar weights.

\subsection{Text Classification}

\textbf{Experimental Setups.}
For the text classification task, we utilize the AG News corpus, a widely used benchmark dataset containing news articles categorized into four distinct classes. The classification is performed using a transformer-based architecture, which is known for its effectiveness in handling natural language processing tasks.
To preprocess the dataset, we tokenize the entire corpus using the \textit{basic\_english} tokenizer. Any words not found in the vocabulary are replaced with a special token, \textit{UNK}, to handle out-of-vocabulary terms. This preprocessing ensures that the dataset is standardized and ready for training.
We save intermediate model checkpoints throughout the training process, starting from the initial stages. From this set of checkpoints, we systematically select every 100th checkpoint for inclusion in the averaging process. The hyperparameter $K$ controls the total number of checkpoints used for averaging, allowing flexible experimentation with different levels of checkpoint aggregation. This design enables us to thoroughly evaluate the impact of checkpoint averaging on the model’s performance.


\textbf{Results.}
As shown in Figure \ref{fig:ag}, the improvement of weight averaging over the SGD baseline is minimal for relatively simple tasks, primarily serving to stabilize training. However, our SeWA method consistently achieves the best results regardless of task complexity, demonstrating its broad applicability across diverse settings.

%% file: sections/6_Con.tex
\section{Conclusion}
\label{sec:Con}
We propose a new algorithm \method{} for adaptive selecting checkpoints to average, which balances generalization performance and convergence speed. Under different function assumptions, we derive its generalization bounds, exhibiting superior results compared to other algorithms. In practical implementation, we employ probabilistic reparameterization to transform the discrete optimization problem into a continuous objective solvable by gradient-based methods. Empirically, we verify that our approach can help to obtain good performance for unstable training processes, and a few checkpoints selected by \method{} can achieve results due to other algorithms using several times as many points. 

%% file: sections/7_App.tex
\section{Experiment Details}
\label{sec:ExpDetail}

\subsection{Behavior Cloning}

\textbf{Network Architecture.} The network architecture comprises four layers, each consisting of a sequence of ReLU activation, Dropout for regularization, and a Linear transformation. The final layer includes an additional Tanh activation function to enhance the representation and capture non-linearities in the output.

\textbf{Results.} 
Comprehensive results for each task across all datasets are presented in Table \ref{tab:BC}. Our evaluation focuses specifically on the medium and medium-expert datasets, which offer a balanced mix of trajectories with varying performance levels. This selection enables a thorough assessment of our method's ability to generalize across different reward distributions.
For clarity and ease of comparison, the main paper emphasizes the average performance across tasks, as illustrated in Figure \ref{fig:BC}. 

\begin{table*}[t!]
\centering
\caption{
Performance comparison of various methods on D4RL Gym tasks. The left panel shows results obtained using the final checkpoint under different update strategies, while the right panel presents results from averaged checkpoints collected during the final training stage with SGD, using different selection strategies. 
Each result is evaluated as the mean of 60 random rollouts, based on 3 independently trained models with 20 trajectories per model.
}
\label{tab:BC}
\scalebox{0.9}{
\begin{tabular}{c|cc|ccc|ccc}
\toprule[2pt]
\multicolumn{1}{l|}{} & Task & Dataset & SGD & SWA & EMA & LAWA & Random & \method{} (Ours) \\ \midrule
\multirow{7}{*}{K=10} & Hopper & medium & 1245.039 & 1279.249 & 1297.270 & 1289.515 & 1291.478 & \textbf{1324.848} \\
 & Hopper & medium-expert & 1460.785 & 1468.893 & 1320.408 & 1462.452 & 1451.015 & \textbf{1509.317} \\
 & Walker2d & medium & 3290.248 & 3328.121 & 3341.888 & 3341.437 & 3306.763 & \textbf{3371.202} \\
 & Walker2d & medium-expert & 3458.693 & 3546.008 & 3681.504 & 3634.373 & 3609.611 & \textbf{3679.806} \\
 & Halfcheetah & medium & 4850.490 & 4858.224 & 4894.204 & 5012.389 & 4896.104 & \textbf{5041.369} \\
 & Halfcheetah & medium-expert & 5015.689  & 4974.923 & 4857.562 & 4989.329 & 4962.719 & \textbf{5082.902} \\ \cmidrule{2-9}
 & \multicolumn{2}{c|}{Average} & 3220.157 & 3242.570 & 3232.139 & 3288.249 & 3252.948 & \textbf{3334.907} \\ \midrule
\multirow{7}{*}{K=20} & Hopper & medium & 1245.039 & 1281.910 & 1302.400 & 1310.875 & 1312.166 & \textbf{1361.202} \\
 & Hopper & medium-expert & 1460.785 & 1427.47 & 1373.268 & 1563.307 & 1482.012 & \textbf{1571.127} \\
 & Walker2d & medium & 3290.248 & 3308.464 & \textbf{3420.257} & 3325.873 & 3324.557 & 3364.886 \\
 & Walker2d & medium-expert & 3458.693 & 3588.176 & 3667.809 & 3557.925 & 3650.846 & \textbf{3673.804} \\
 & Halfcheetah & medium & 4850.490 & 4913.549 & 4848.006 & 4974.041 & 4924.613 & \textbf{5071.051} \\
 & Halfcheetah & medium-expert & 5015.689 & 5024.723 & 4957.194 & 4993.524 & 4988.816 & \textbf{5085.628} \\ \cmidrule{2-9}
 & \multicolumn{2}{c|}{Average} & 3220.157 & 3257.382 & 3261.489 & 3287.591 & 3280.502 & \textbf{3354.616} \\ \midrule
\multirow{7}{*}{K=50} & Hopper & medium & 1245.039 & 1294.884 & 1329.863 & 1336.33 & 1319.571 & \textbf{1389.280} \\
 & Hopper & medium-expert & 1460.785 & 1477.466 & 1485.696 & 1537.672 & 1496.045 & \textbf{1616.116} \\
 & Walker2d & medium & 3290.248 & 3262.046 & 3341.767 & 3253.695 & 3352.12 & \textbf{3392.130} \\
 & Walker2d & medium-expert & 3458.693 & 3577.509 & 3591.081 & 3584.468 & 3659.789 & \textbf{3672.560} \\
 & Halfcheetah & medium & 4850.490 & 4927.951 & 4968.048 & 5022.097 & 5000.004 & \textbf{5035.631} \\
 & Halfcheetah & medium-expert & 5015.689  & 5061.688 & \textbf{5075.426} & 5011.232 & 4960.585 & 5044.886 \\ \cmidrule{2-9}
 & \multicolumn{2}{c|}{Average} & 3220.157 & 3280.833 & 3298.647 & 3290.916 & 3298.019 & \textbf{3358.434} \\ \midrule
\multirow{7}{*}{K=100} & Hopper & medium & 1245.039 & 1347.267 & 1322.625 & 1320.652 & 1319.727 & \textbf{1393.981} \\
 & Hopper & medium-expert & 1460.785 & 1527.206 & 1528.265 & 1496.266 & 1491.196 & \textbf{1568.025} \\
 & Walker2d & medium & 3290.248 & 3324.218 & 3393.646 & 3345.913 & 3321.046 & \textbf{3424.078} \\
 & Walker2d & medium-expert & 3458.693 & 3575.621 & 3629.308 & 3613.274 & 3587.211 & \textbf{3710.347} \\
 & Halfcheetah & medium & 4850.490 & 4939.629 & 4871.376 & 4974.220 & 5015.349 & \textbf{5021.948} \\
 & Halfcheetah & medium-expert & 5015.689 & 4919.624 & 5047.757 & 4991.007 & 5031.975 & \textbf{5063.546} \\ \cmidrule{2-9}
 & \multicolumn{2}{c|}{Average} & 3220.157 & 3272.261 & 3298.830 & 3290.222 & 3294.417 & \textbf{3363.654} \\ 
 \bottomrule
\end{tabular}
}
\end{table*}

\subsection{Image Classification}

\textbf{Network Architecture.} The network architecture consists of three primary blocks, followed by an average pooling layer and a linear layer for generating the final output. Each block contains two convolutional layers, each accompanied by a corresponding batch normalization layer to improve training stability and convergence. To address potential issues of vanishing gradients, each block includes a shortcut connection that facilitates efficient gradient flow during backpropagation. The output of each block is passed through a ReLU activation function to introduce non-linearity, enabling the network to learn complex representations effectively.

\textbf{Results.} In addition to the results presented in Figure \ref{fig:cifar100}, we provide further analysis examining the impact of network parameter variations to demonstrate the robustness of our method across networks of different sizes. These results, shown in Figure \ref{fig:cifar100-layers}, illustrate that as the number of layers or blocks increases, the performance of SGD improves, following a similar training curve.

Notably, weight averaging consistently outperforms SGD during the upward phase of training. The performance gains from weight averaging become more pronounced as the network size increases, highlighting its potential in scaling effectively to larger models. This highlights the potential of weight averaging to enhance the performance of larger models. 
Furthermore, regardless of changes in network parameters, our proposed method consistently achieves superior results, demonstrating its adaptability and effectiveness across varying network configurations. These findings emphasize the potential of weight averaging as a robust and scalable technique for optimizing model performance.

\begin{figure*}[t!]
    \centering
    \subfigure{
    \centering
    \includegraphics[width=0.3\textwidth]{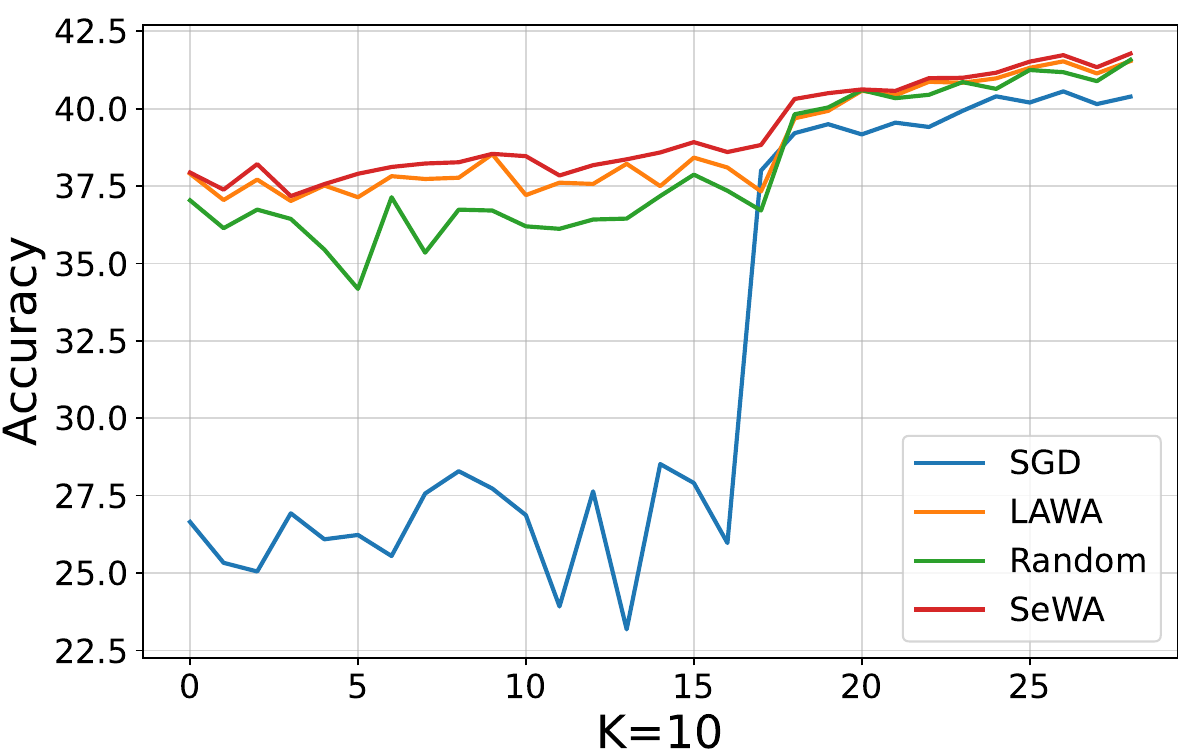}}
    \centering
    \subfigure{
    \centering
    \includegraphics[width=0.3\textwidth]{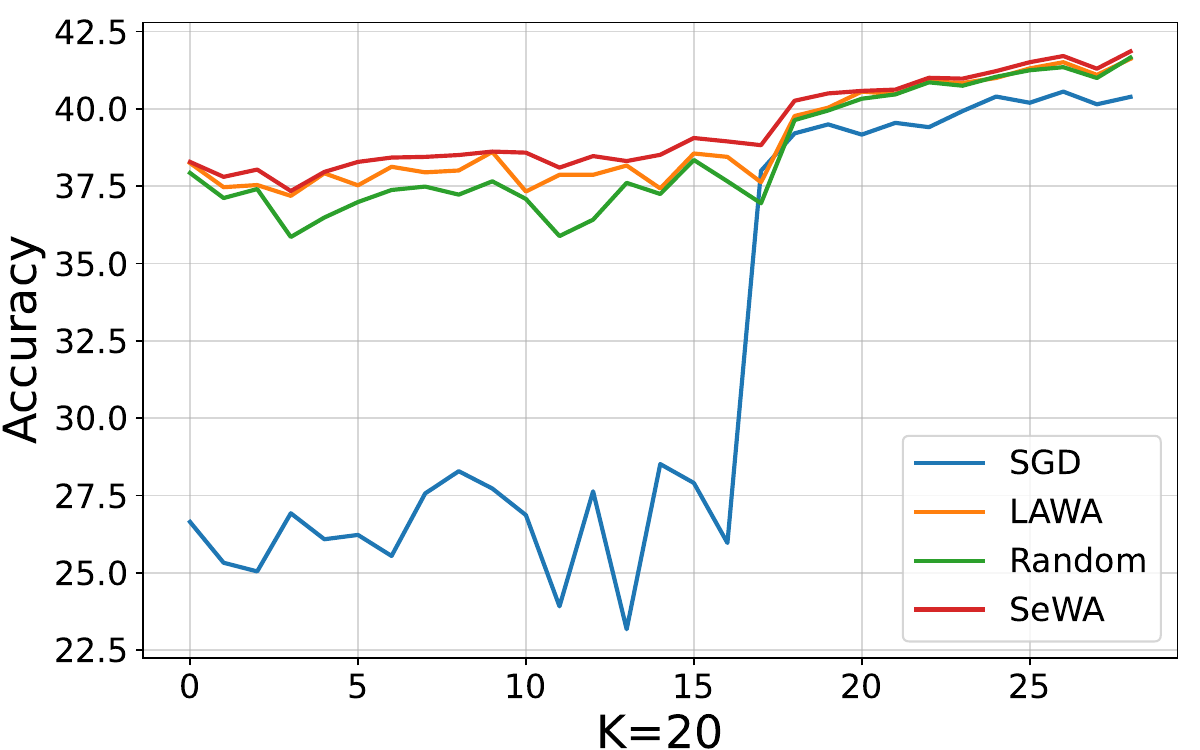}}
    \subfigure{
    \centering
    \includegraphics[width=0.3\textwidth]{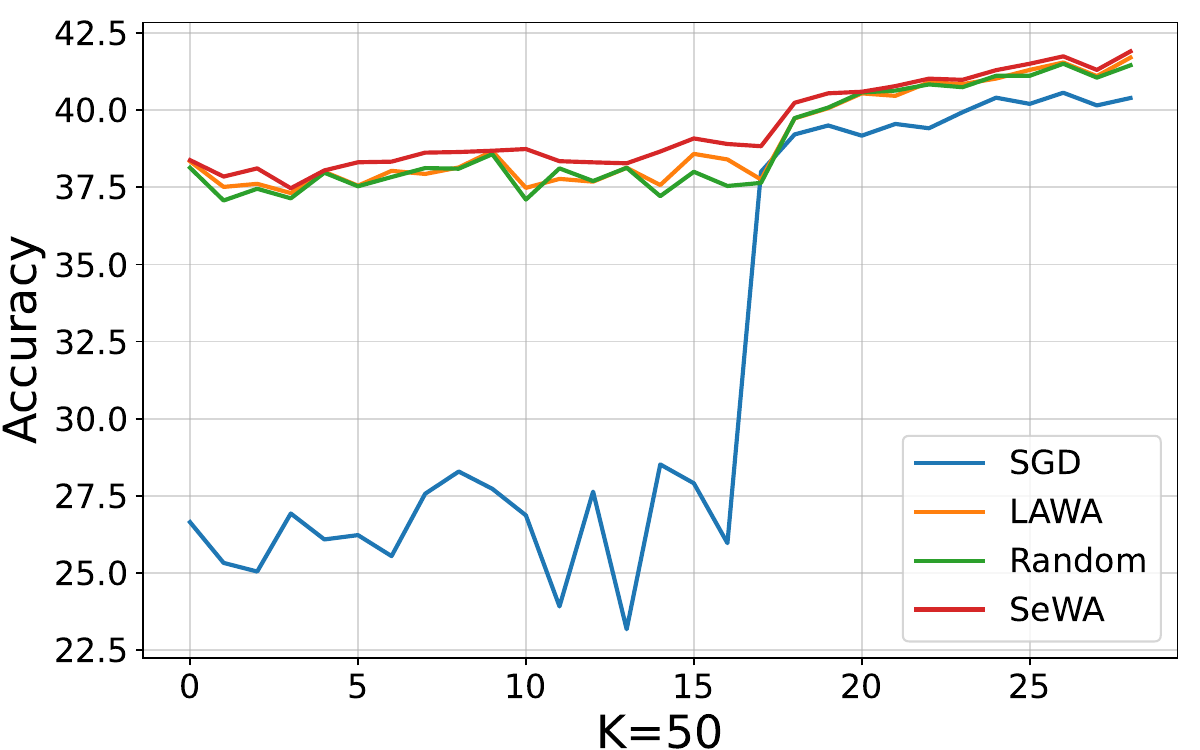}}
    \subfigure{
    \centering
    \includegraphics[width=0.3\textwidth]{figs/Cifar100-ResNet3-k10.pdf}}
    \centering
    \subfigure{
    \centering
    \includegraphics[width=0.3\textwidth]{figs/Cifar100-ResNet3-k20.pdf}}
    \subfigure{
    \centering
    \includegraphics[width=0.3\textwidth]{figs/Cifar100-ResNet3-k50.pdf}}
    \subfigure{
    \centering
    \includegraphics[width=0.3\textwidth]{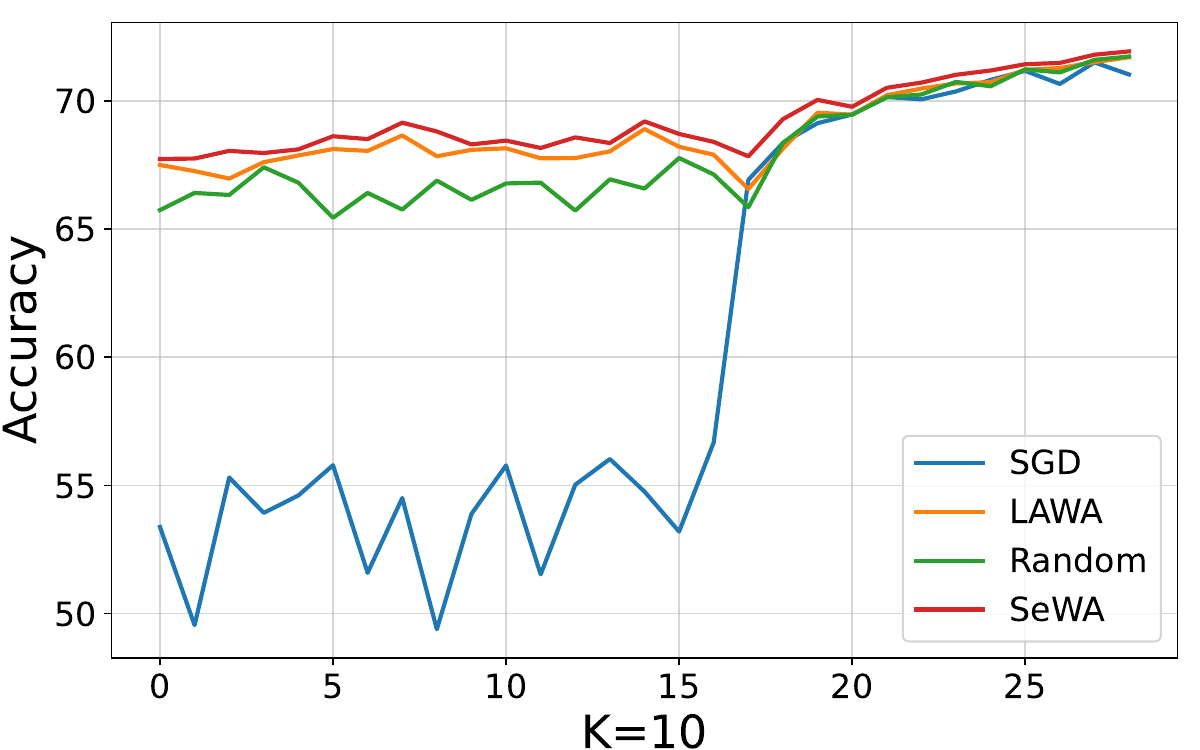}}
    \centering
    \subfigure{
    \centering
    \includegraphics[width=0.3\textwidth]{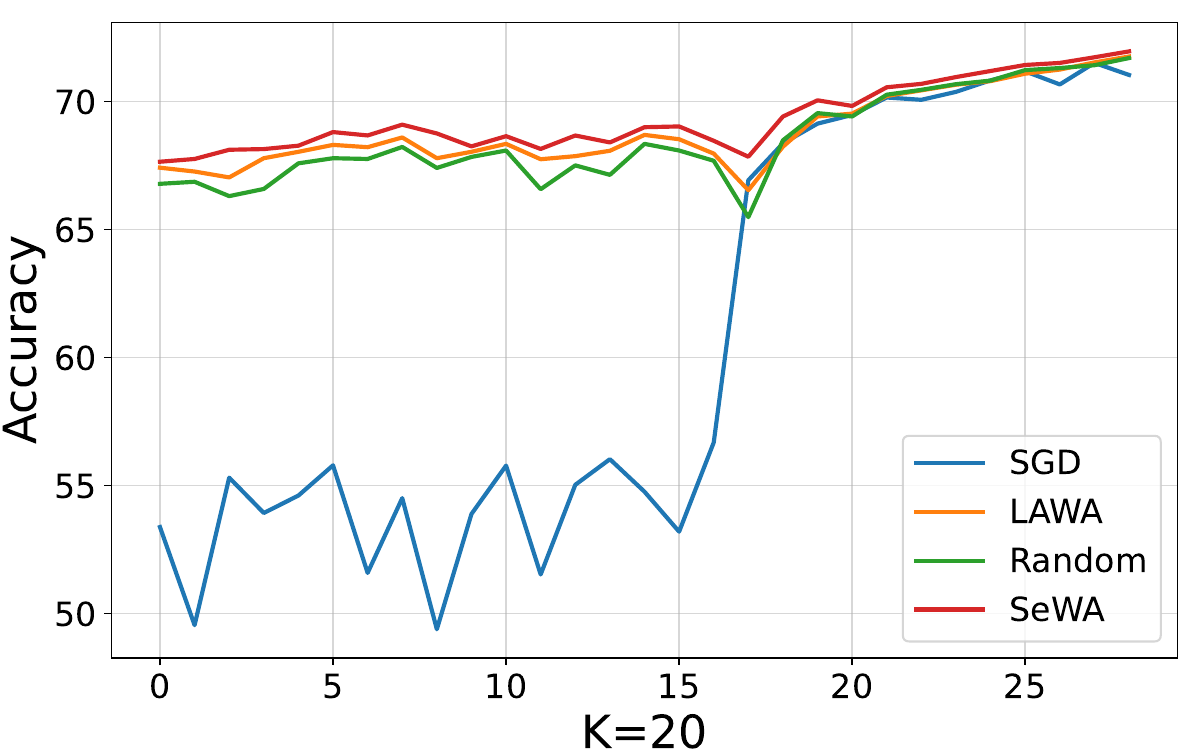}}
    \subfigure{
    \centering
    \includegraphics[width=0.3\textwidth]{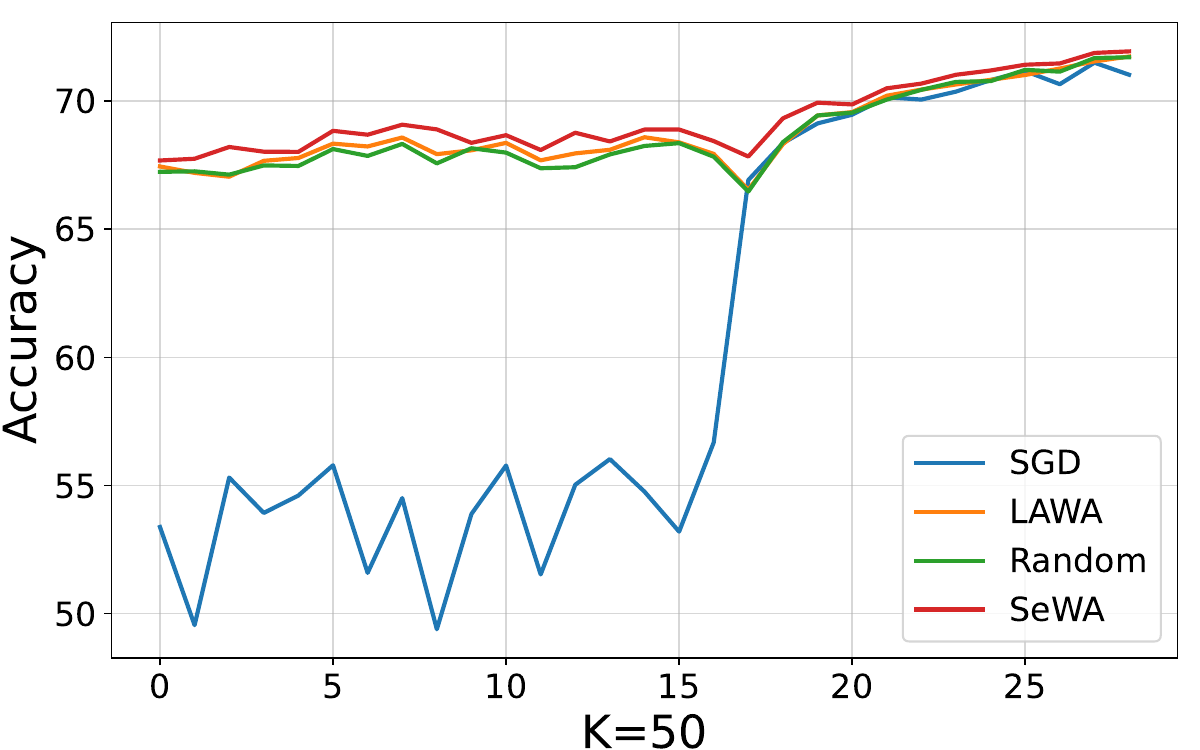}}
    \vspace{-0.2cm}
    \caption{From left to right, the figures illustrate the impact of the hyperparameter $K$ on the CIFAR-100 task. 
    Each data point represents performance based on intervals of 100 checkpoints, with $K$ checkpoints selected from these intervals using various strategies.
    The first row corresponds to a network architecture with 1 block, the second row represents a network with 3 blocks, and the third row depicts results for a network with 5 blocks.
    }
    \label{fig:cifar100-layers}
\end{figure*}

\subsection{Text Classification}

\textbf{Network Architectures.} The network architecture comprises two embedding layers followed by two layers of \textit{TransformerEncoderLayer}. Each \textit{TransformerEncoderLayer} includes a multi-head self-attention mechanism and a position-wise feedforward network, along with layer normalization and residual connections to enhance training stability and gradient flow. The output from the Transformer layers is passed through a linear layer to produce the final predictions.

\textbf{Results.} In addition to the findings presented in Figure \ref{fig:ag}, we conduct further analysis to evaluate the impact of network parameter variations, demonstrating the robustness of our method across networks of varying sizes. These additional results, shown in Figure \ref{fig:ag-layers}, indicate that as the number of Transformer layers increases, the performance of SGD improves up to a certain point. However, beyond this range - where two layers appear sufficient - performance begins to exhibit fluctuations, suggesting diminishing returns and instability with additional layers.

While the improvement achieved by weight averaging is relatively modest due to the simplicity of the task, it still plays a critical role in stabilizing the training process and reducing fluctuations in the training curve. Among the averaging methods evaluated, our proposed method consistently achieves the best performance, underscoring its effectiveness in maintaining stability and optimizing performance, even in scenarios where task complexity is low.

\begin{figure*}[t!]
    \centering
    \subfigure{
    \centering
    \includegraphics[width=0.3\textwidth]{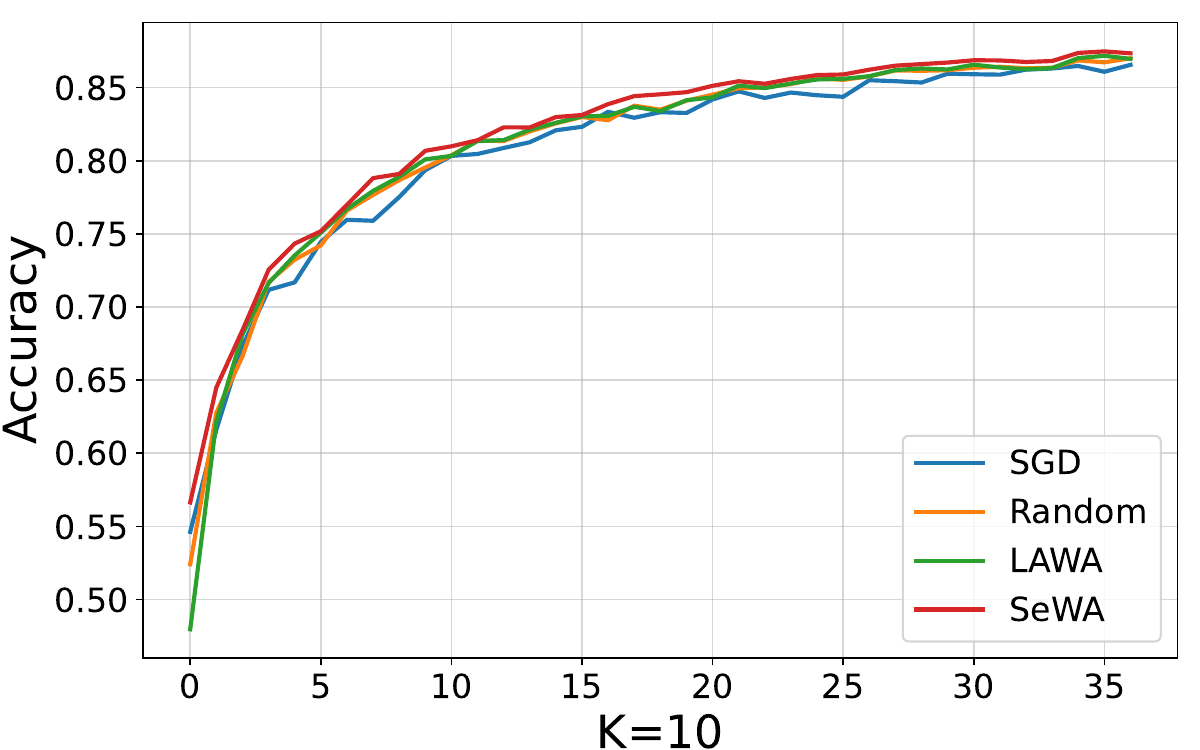}}
    \centering
    \subfigure{
    \centering
    \includegraphics[width=0.3\textwidth]{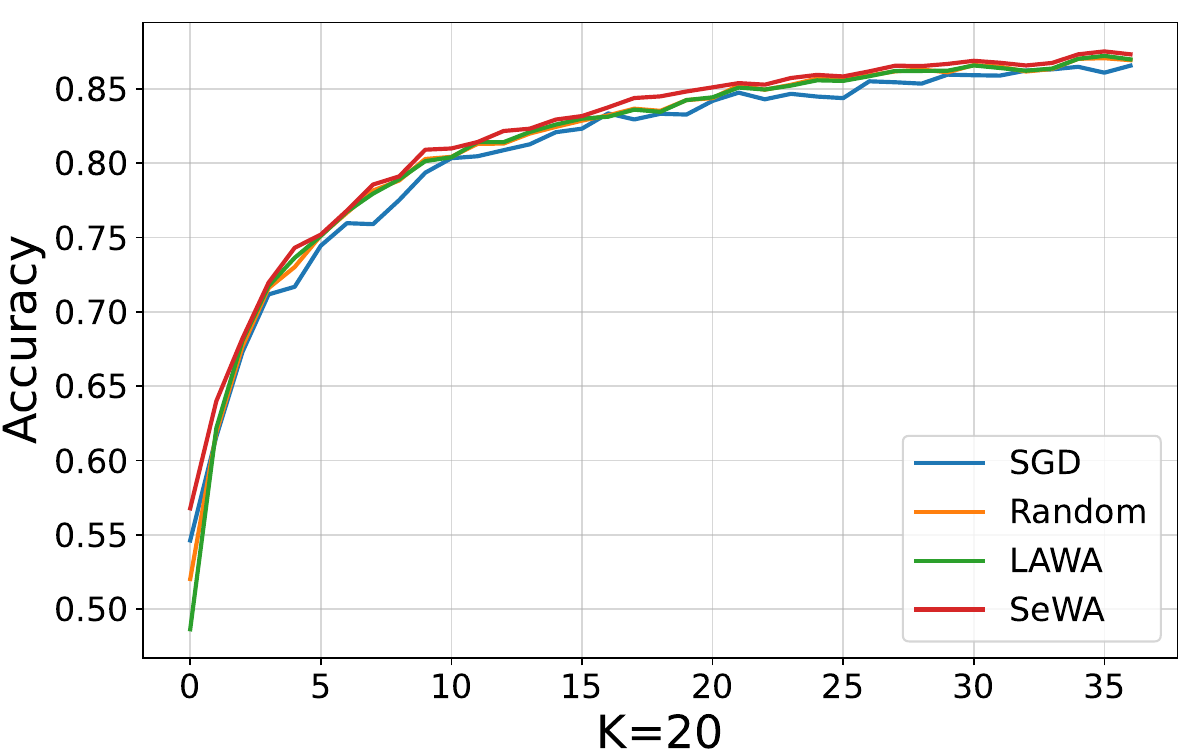}}
    \subfigure{
    \centering
    \includegraphics[width=0.3\textwidth]{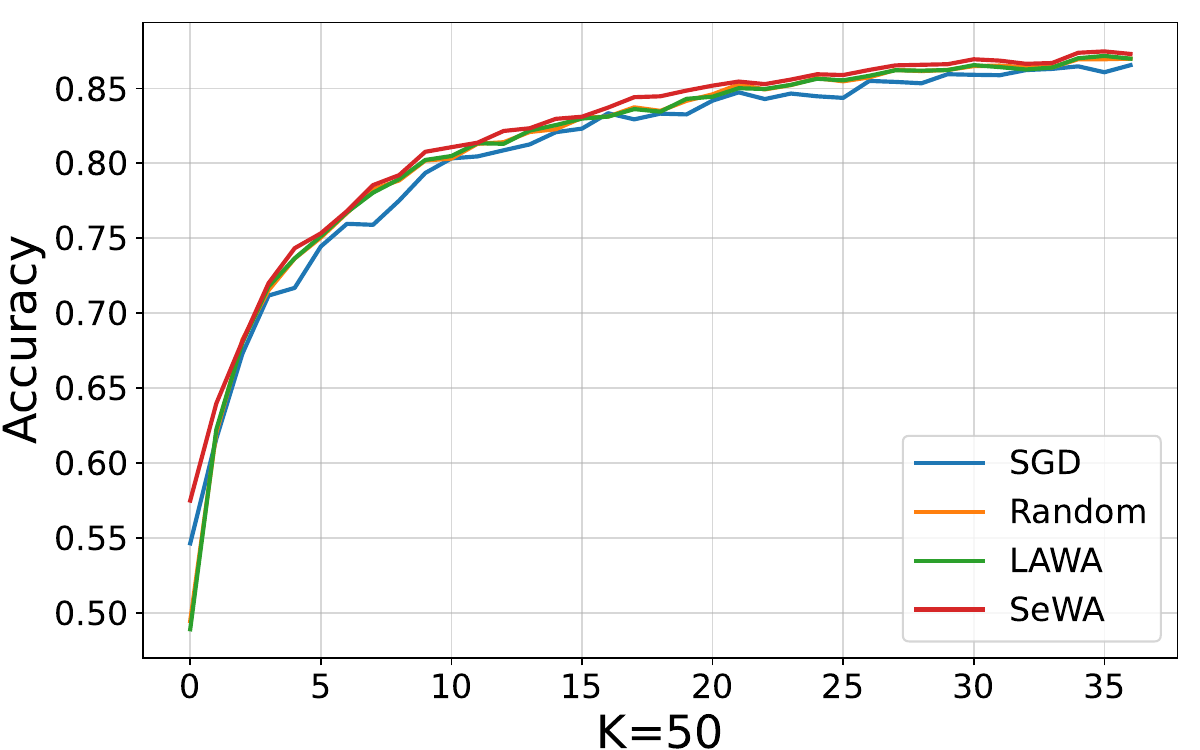}}
    \subfigure{
    \centering
    \includegraphics[width=0.3\textwidth]{figs/NLP-layer2-k10.pdf}}
    \centering
    \subfigure{
    \centering
    \includegraphics[width=0.3\textwidth]{figs/NLP-layer2-k20.pdf}}
    \subfigure{
    \centering
    \includegraphics[width=0.3\textwidth]{figs/NLP-layer2-k50.pdf}}
    \subfigure{
    \centering
    \includegraphics[width=0.3\textwidth]{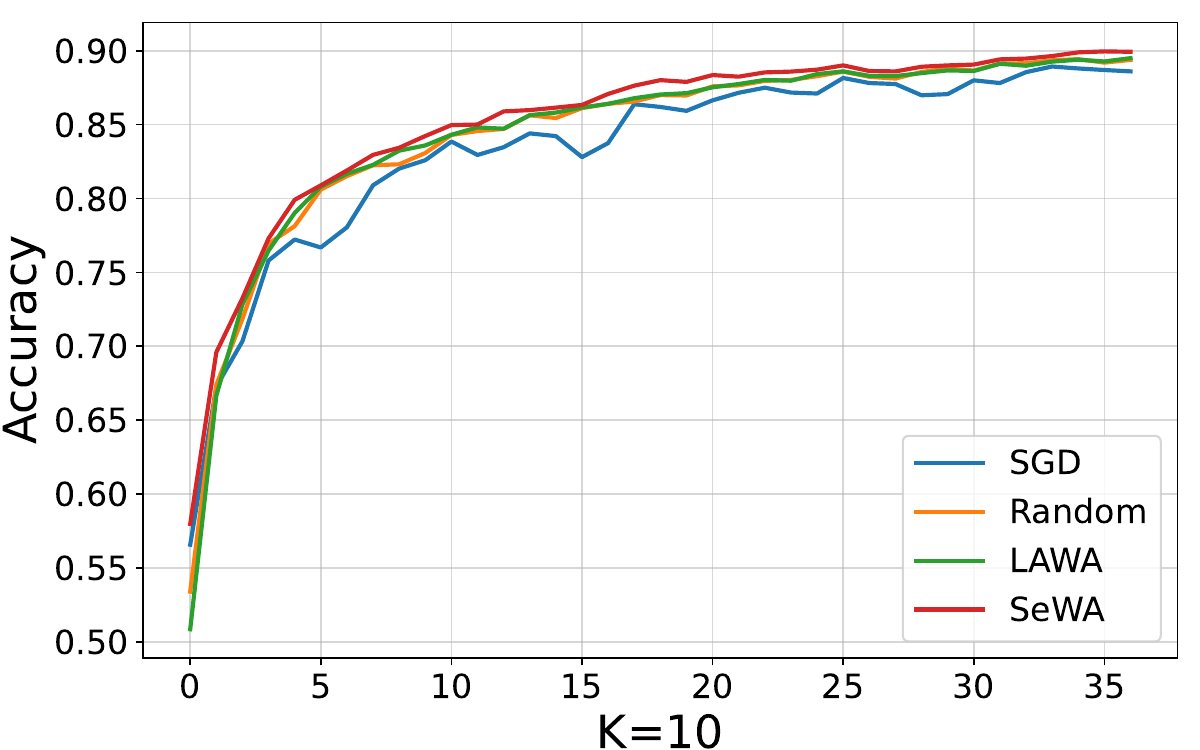}}
    \centering
    \subfigure{
    \centering
    \includegraphics[width=0.3\textwidth]{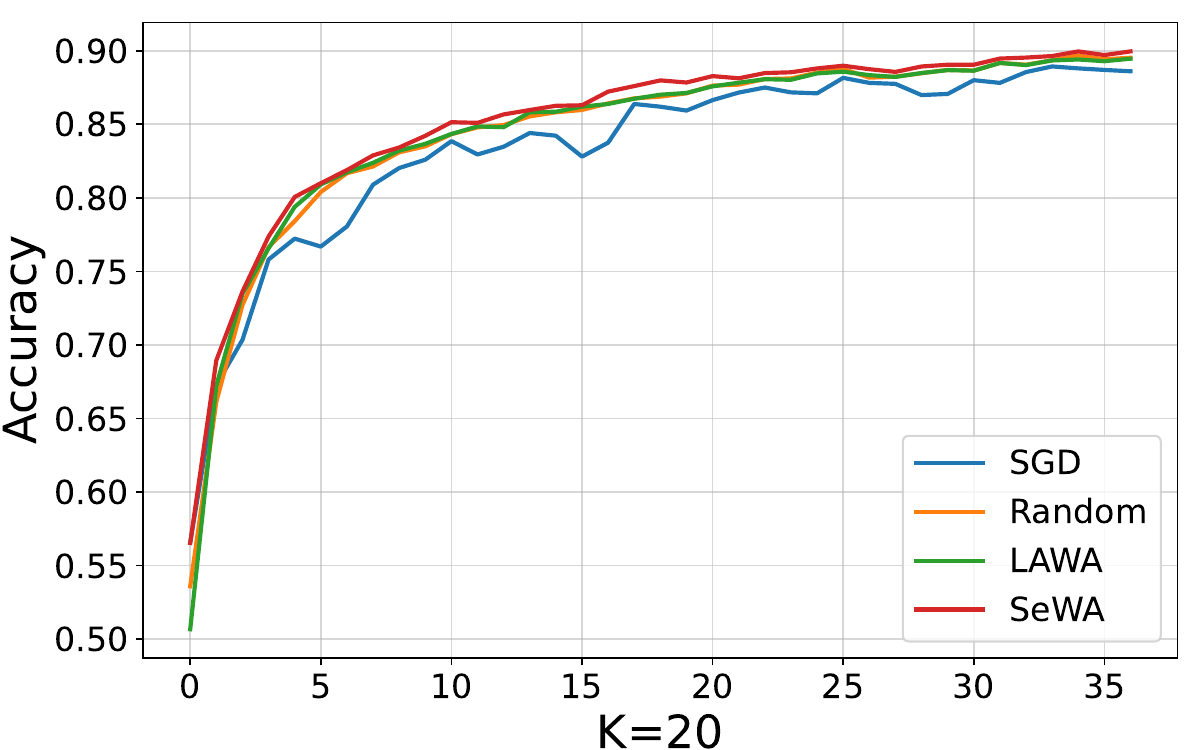}}
    \subfigure{
    \centering
    \includegraphics[width=0.3\textwidth]{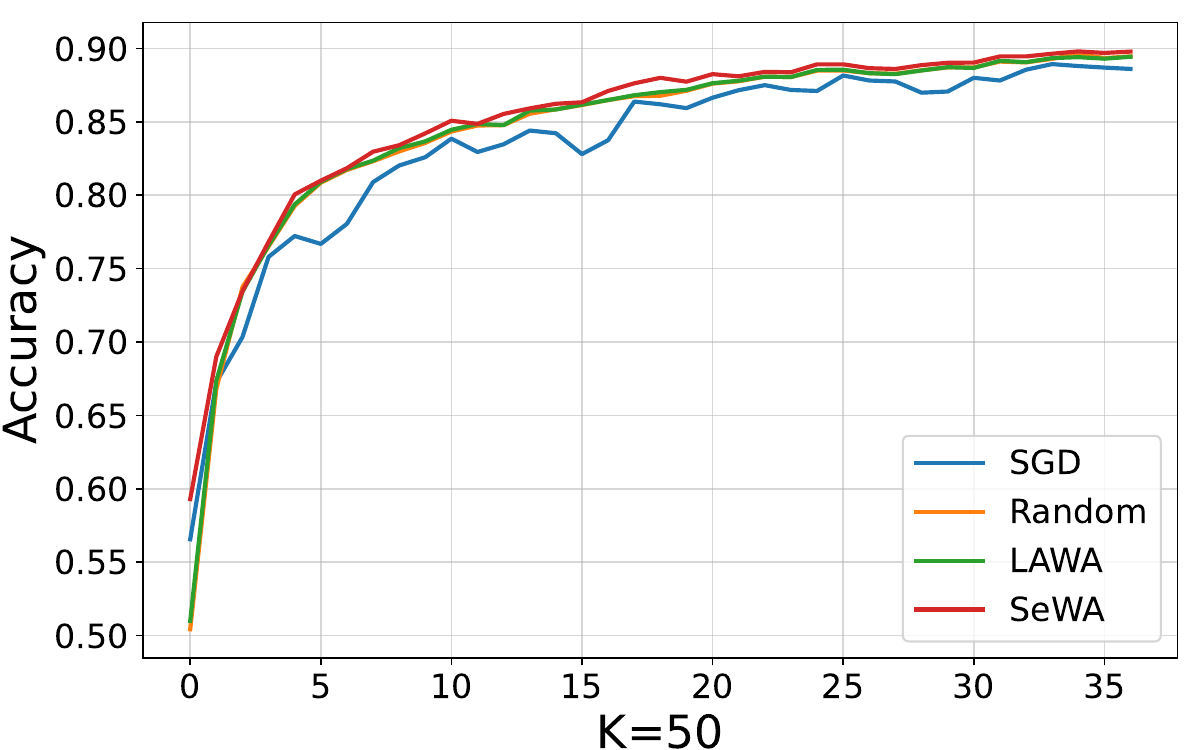}}
    \vspace{-0.2cm}
    \caption{From left to right, the figures illustrate the impact of the hyperparameter $K$ on the AG News corpus. Each point corresponds to intervals of 100 checkpoints, with $K$ checkpoints selected from these intervals using different strategies.
    The first row corresponds to a network architecture with a single \textit{TransformerEncoderLayer}, the second row represents a network with three \textit{TransformerEncoderLayer}s, and the third row shows results for a network with five \textit{TransformerEncoderLayer}s.
    }
    \label{fig:ag-layers}
\end{figure*}

\section{Proof of Lemma \ref{lemma}}\label{pro-lemma}  \paragraph{\boldmath$(1+\alpha\beta)$-expansive.} According to triangle inequality and $\beta$-smoothness,
\begin{equation}\label{eq:app2.1}
     \begin{aligned}
      \Vert w_{T+1} - w_{T+1}^{\prime}\Vert &\leq \Vert w_T- w_T^{\prime}\Vert + \alpha\Vert \nabla F(w_T) -\nabla F(w_T^{\prime})\Vert \\
      &\leq \Vert w_T- w_T^{\prime}\Vert + \alpha\beta \Vert w_T- w_T^{\prime}\Vert \\
      &= (1+\alpha\beta)\Vert w_T- w_T^{\prime}\Vert .
     \end{aligned}
\end{equation}

\paragraph{\emph{Non}-expansive.} Function is convexity and $\beta$-smoothness that implies 
\begin{equation}\label{eq:app2.2}
     \begin{aligned}
      \langle \nabla F(w) -\nabla F(v), w - v \rangle \geq \frac{1}{\beta} \Vert \nabla F(w) -\nabla F(v)\Vert^2 .
     \end{aligned}
\end{equation}
We conclude that
\begin{equation}\label{eq:app2.3}
     \begin{aligned}
      \Vert w_{T+1} - w_{T+1}^{\prime}\Vert &= \sqrt{\Vert w_{T} - \alpha \nabla F(w_{T}) - w_{T}^{\prime} + \alpha \nabla F(w_{T}^{\prime})\Vert^2}  \\
      &=\sqrt{\Vert w_{T} - w_{T}^{\prime} \Vert^2 - 2\alpha\langle \nabla F(w_{T}) -\nabla F(w_{T}^{\prime}), w_T- w_T^{\prime} \rangle +\alpha^2 \Vert \nabla F(w_{T}) - \nabla F(w_{T}^{\prime})\Vert^2} \\
      &\leq \sqrt{\Vert w_T- w_T^{\prime}\Vert^2 - \left(\frac{2\alpha}{\beta} -\alpha^2 \right) \Vert \nabla F(w_{T}) -\nabla F(w_{T}^{\prime})\Vert^2} \\
      &\leq \Vert w_T- w_T^{\prime}\Vert.
     \end{aligned}
\end{equation}

\section{Proof of the generalization bounds}\label{pro-con}
By the Lemma \ref{convex-basic} and \ref{nonconvex-basic}, the proof of Theorem \ref{thm:stability-conv} and \ref{thm:stability-non-with} can be further decomposed into bounding the difference of the parameters for the last $k$ points of the average algorithm. We provide the proof as follows. And you can also find it in \cite{peng2020dfwa}.

\subsection{Update rules of the last $k$ points of the averaging algorithm.}
For the last $k$ points of the averaging algorithm, we formulate it as
\begin{equation}\label{FWA-rules}
    \hat{w}^{k}_{T}=\frac{1}{k} \sum_{i=T-k+1}^{T} w_{i}.
\end{equation}
It is not difficult to find the relationship between $\bar{w}^{k}_{T}$ and $\bar{w}^{k}_{T-1}$, i.e.,
\begin{equation}\label{pro-FWA-update}
    \hat{w}^{k}_{T} = \hat{w}^{k}_{T-1} + \frac{1}{k}\left(w_{T} - w_{T-k}\right) = \hat{w}^{k}_{T-1} - \frac{1}{k}\sum_{i=T-k+1}^{T} \alpha_i\nabla F(w_{i-1},z_i),
\end{equation}
where the second equality follows from the update of SGD.

\subsection{\textbf{Proof. Theorem \ref{thm:stability-conv}}}\label{proof-thm-con-with}
  First, using the relationship between $\hat{w}^{k}_{T}$ and $\hat{w}^{k}_{T-1}$ in Eq. ~\eqref{pro-FWA-update}, we consider that the different sample $z_{T}$ and $z_{T}^{\prime}$ are selected to update with probability $\frac{1}{n}$ at the step $T$.  
\begin{equation}
  \begin{aligned}
   \bar{\delta}_{T} &= \bar{\delta}_{T-1} + \frac{1}{k}\sum_{i=T-k+1}^{T} \alpha_i \Vert\nabla F(w^{\prime}_{i-1},z_i) - \nabla F(w_{i-1},z_i) \Vert \\
   &\leq \bar{\delta}_{T-1} + \frac{2\alpha_T L}{k} + \frac{1}{k}\sum_{i=T-k+1}^{T-1} \alpha_i \Vert\nabla F(w^{\prime}_{i-1},z_i) - \nabla F(w_{i-1},z_i) \Vert ,
  \end{aligned}
 \end{equation}
where the proof follows from the triangle inequality and $L$-Lipschitz condition. For $\frac{1}{k}\sum_{i=T-k+1}^{T-1} \alpha_i \Vert\nabla F(w^{\prime}_{i-1},z_i) - \nabla F(w_{i-1},z_i) \Vert$ will be controlled in the late.

Second, another situation need be considered in case of the same sample are selected$(z_{T}=z_{T}^{\prime})$ to update with probability $1-\frac{1}{n}$ at the step $T$. 
\begin{equation}
  \begin{aligned}
   \bar{\delta}_{T} &= \bar{\delta}_{T-1} + \frac{1}{k}\sum_{i=T-k+1}^{T} \alpha_i \Vert\nabla F(w^{\prime}_{i-1},z_i) - \nabla F(w_{i-1},z_i) \Vert \\
   &\leq \bar{\delta}_{T-1} + \frac{1}{k}\sum_{i=T-k+1}^{T-1} \alpha_i\Vert\nabla F(w^{\prime}_{i-1},z_i) - \nabla F(w_{i-1},z_i) \Vert ,
  \end{aligned}
 \end{equation}
where $\Vert\nabla F(w^{\prime}_{T-1},z_T)-\nabla F(w_{T-1},z_T)\Vert=0$ in the second inequality because the non-expansive property of convex function.

For each $\Vert\nabla F(w^{\prime}_{i-1},z_i)-\nabla F(w_{i-1},z_i)\Vert$ in the sense of expectation, We consider two situations using $\alpha L$ bound and the non-expansive property. Then, we get  
  \begin{equation}
    \frac{1}{k}\sum_{i=T-k+1}^{T-1}\alpha_i \Vert\nabla F(w^{\prime}_{i-1},z_i) - \nabla F(w_{i-1},z_i) \Vert \leq \frac{2L}{nk}\sum_{i=T-k+1}^{T-1}\alpha_i.
 \end{equation}

Then we obtain the expectation based on the above analysis 
  \begin{equation}
  \begin{aligned}
    \mathbb{E}\left[\bar{\delta}_{T}\right] &\leq (1-\frac{1}{n})\bar{\delta}_{T-1} + \frac{1}{n}\left(\bar{\delta}_{T-1}+\frac{2\alpha_T L}{k}\right) + \frac{2L}{nk}\sum_{i=T-k+1}^{T-1}\alpha_i\\
    &\leq \mathbb{E}\left[\bar{\delta}_{T-1}\right] + \frac{2L}{nk}\sum_{i=T-k+1}^{T}\alpha_i
  \end{aligned}
 \end{equation}
recursively, we can get 
    \begin{equation}
     \begin{aligned}
      \mathbb{E}\left[\bar{\delta}_{T}\right]&\leq \frac{2L}{nk} \left( \sum_{i=T-k+1}^{T}\alpha_i + \sum_{i=T-k}^{T-1}\alpha_i + \cdots + \sum_{i=1}^{k}\alpha_i \right) \\ & + \frac{2L}{nk} \left( \sum_{i=1}^{k-1}\alpha_i + \sum_{i=1}^{k-2}\alpha_i + \cdots + \sum_{i=1}^{1}\alpha_i \right). \\
     \end{aligned}
    \end{equation}
Let $\alpha_{i,j}=\alpha$, we get
    \begin{equation}
     \begin{aligned}
      \mathbb{E}\left[\bar{\delta}_{T}\right] = \frac{2\alpha L}{n} \left( T - \frac{k}{2} \right).
     \end{aligned}
    \end{equation}
Plugging this back into Eq.~\eqref{convex-basic}, we obtain
 \begin{equation}\label{eq:2.2.1}
  \epsilon_{gen} = \mathbb{E}\vert F(\bar{w}_T^K;z)-F(\bar{w}^{K\prime}_T;z)\vert \leq \frac{2\alpha L^2 s}{n} \left( T - \frac{k}{2} \right).
 \end{equation}
And we finish the proof.

\subsection{Proof of Lemma \ref{nonconvex-basic}}\label{proof-noncon-basic}
We consider that $S$ and $S^\prime$ are two samples of size $n$ differing in only a single example. Let $\xi$ denote the event $\bar{\delta}_{t_0}=0$. Let $z$ be an arbitrary example and consider the random variable $I$ assuming the index of the first time step using the different sample. then we have
    \begin{equation}
     \begin{aligned}
      \mathbb{E}\vert \nabla F(\bar{w}_T^{K};z)-\nabla F(\bar{w}^{K\prime}_T;z)\vert &= P\left\lbrace \xi\right\rbrace \mathbb{E}[\vert \nabla F(\bar{w}_T^{K};z)-\nabla F(\bar{w}^{K\prime}_T;z)\vert|\xi]\\
      &+P\left\lbrace \xi^{c}\right\rbrace E[\vert \nabla F(\bar{w}_T^{K};z)-\nabla F(\bar{w}^{K\prime}_T;z)\vert |\xi^{c}]\\
      &\leq P\left\lbrace I\geq t_0\right\rbrace \cdot \mathbb{E}[\vert \nabla F(\bar{w}_T^{K};z)-\nabla F(\bar{w}^{K\prime}_T;z)\vert |\xi] \\
      &+P\left\lbrace I\leq t_0\right\rbrace \cdot \mathop{sup}_{\bar{w}^{K},z} F(\bar{w}^{K};z),\\
     \end{aligned}
    \end{equation}
where $\xi^{c}$ denotes the complement of $\xi$.   

Note that Note that when $I\geq t_0$, then we must have that $\bar{\delta}_{t_0}=0$, since the execution on $S$ and $S^{\prime}$ is identical until step $t_0$. We can get $LE[\Vert\bar{w}_{T}^{K} - \bar{w}_{T}^{K\prime}\Vert|\xi]$ combined the Lipschitz continuity of $F$. Furthermore, we know $P\left\lbrace \xi^{c}\right\rbrace=P\left\lbrace \bar{\delta}_{t_0}=0\right\rbrace\leq P\left\lbrace I\leq t_0\right\rbrace$, for the random selection rule, we have 
    \begin{equation}
     \begin{aligned}
      P\left\lbrace I\leq t_0\right\rbrace \leq \sum_{t=1}^{t_0} P\left\lbrace I = t_0\right\rbrace = \frac{t_0}{n}.
     \end{aligned}
    \end{equation}
We can combine the above two parts and $F \in [0,1]$ to derive 
the stated bound $L\mathbb{E}[\Vert\bar{w}_{T}^{k} - \bar{w}_{T}^{k\prime}\Vert\vert\xi]+\frac{t_0}{n}$, which completes the proof.

\subsection{Lemma \ref{Lemma_noncon} and it's proof}\label{proof-Lemma_noncon}
\begin{lemma}\label{Lemma_noncon}
Assume that $F$ is $\beta$-smooth and $non$-convex. Let $\alpha = \frac{c}{T}$, we have 
  \begin{equation}
  \begin{aligned}
   \Vert w^{\prime}_{T}& - w_{T} \Vert \leq &e^\frac{c\beta k}{T}\bar{\delta}_{T},
  \end{aligned}
 \end{equation}
\end{lemma}
where $\bar{\delta}_{T}= \frac{1}{k} \sum_{i=T-k+1}^{T}\Vert w^{\prime}_{i} - w_{i} \Vert$.

\textbf{proof Lemma \ref{Lemma_noncon}.} 
By triangle inequality and our assumption that $F$ satisfies, we have
 \begin{equation}
  \begin{aligned}
   \Vert w^{\prime}_{T}& - w_{T} \Vert = \frac{1}{k} \cdot k \cdot \Vert w^{\prime}_{T} - w_{T} \Vert \\
   \leq & \frac{1}{k} ( \Vert w^{\prime}_{T} - w_{T} \Vert + (1+\alpha_{T-1}\beta)\Vert w^{\prime}_{T-1} - w_{T-1} \Vert + \cdots + \\&(1+\alpha_{T-1}\beta)(1+\alpha_{T-2}\beta)\cdots(1+\alpha_{T-k+1}\beta)\Vert w^{\prime}_{T-k+1} - w_{T-k+1} \Vert ) \\
   \leq & \prod_{t=T-k+1}^{T} (1+\alpha_t\beta)\left(\frac{1}{k} \sum_{i=T-k+1}^{T}\Vert w^{\prime}_{i} - w_{i} \Vert\right).
    \end{aligned}
 \end{equation}
Let $\alpha_t = \alpha = \frac{C}{T}$, we have
  \begin{equation}
  \begin{aligned}
   \Vert w^{\prime}_{T}& - w_{T} \Vert \leq &\prod_{t=T-k+1}^{T} (1+\alpha_t\beta)\bar{\delta}_{T} \leq \left(1+ \frac{c\beta}{T}\right)^k\bar{\delta}_{T} \leq e^\frac{c\beta k}{T}\bar{\delta}_{T}.
  \end{aligned}
 \end{equation}

\subsection{\textbf{Proof. Theorem \ref{thm:stability-non-with} (Based on the constant learning rate)}}
\label{proof-thm-non-with} In the case of non-convex, the $(1+\alpha\beta)$-expansive properties and $L$-Lipschitz conditions are used in our proof. Based on the relationship between $\hat{w}^{k}_T$ and $\hat{w}^{k}_{T-1}$ in Eq. ~\eqref{pro-FWA-update}. We consider that the different samples $z_T$ and $z^{\prime}_T$ are selected to update with probability $\frac{1}{n}$ at step T.
\begin{equation}
  \begin{aligned}
   \bar{\delta}_{T} &= \bar{\delta}_{T-1} + \frac{1}{k}\sum_{i=T-k+1}^{T} \alpha \Vert\nabla F(w^{\prime}_{i-1},z_i) - \nabla F(w_{i-1},z_i) \Vert \\
   &\leq \bar{\delta}_{T-1} + \frac{2\alpha L}{k} + \frac{1}{k}\sum_{i=T-k+1}^{T-1} \alpha \Vert\nabla F(w^{\prime}_{i-1},z_i) - \nabla F(w_{i-1},z_i) \Vert ,
  \end{aligned}
 \end{equation}
Next, the same sample $z=z^{\prime}$ is selected to update with probability $1-\frac{1}{n}$ at step T.
\begin{equation}
  \begin{aligned}
   \bar{\delta}_{T} &= \bar{\delta}_{T-1} + \frac{1}{k}\sum_{i=T-k+1}^{T} \alpha \Vert\nabla F(w^{\prime}_{i-1},z_i) - \nabla F(w_{i-1},z_i) \Vert \\
   &\leq \bar{\delta}_{T-1} + \frac{\alpha \beta}{k}\Vert w^{\prime}_{T-1} - w_{T-1} \Vert + \frac{1}{k}\sum_{i=T-k+1}^{T-1} \alpha \Vert\nabla F(w^{\prime}_{i-1},z_i) - \nabla F(w_{i-1},z_i) \Vert \\
   &\leq (1+\frac{\alpha \beta(1+\alpha \beta)^{k-1}}{k})\bar{\delta}_{T-1} + \frac{1}{k}\sum_{i=T-k+1}^{T-1} \alpha \Vert\nabla F(w^{\prime}_{i-1},z_i) - \nabla F(w_{i-1},z_i) \Vert,
  \end{aligned}
 \end{equation}
where the proof follows from the $\beta$-smooth and Lemma \ref{Lemma_noncon}. Then, we bound the $\alpha \Vert\nabla F(w^{\prime}_{T-2},z_{T-1}) - \nabla F(w_{T-2},z_{T-1}) \Vert$ with different sampling. 
 \begin{equation}\label{noncon-sigbound}
  \begin{aligned}    
    \alpha\Vert\nabla &F(w^{\prime}_{T-2},z_{T-1}) - \nabla F(w_{T-2},z_{T-1}) \Vert = \frac{2\alpha L}{n} + \left(1-\frac{1}{n}\right)\alpha\beta\Vert w_{T-2} - w^{\prime}_{T-2} \Vert\\
    &\leq \frac{2\alpha L}{n} + \alpha\beta\left(\Vert w_{T-3} - w^{\prime}_{T-3} \Vert + \alpha \Vert \nabla F(w^{\prime}_{T-3},z_{T-2})-\nabla F^{\prime}(w_{T-3},z_{T-2}) \Vert\right) \\
    &\leq \frac{2\alpha L}{n} + \alpha\beta\left(\frac{2\alpha L}{n} + (1+\alpha\beta)\Vert w_{T-3} - w^{\prime}_{T-3} \Vert\right) \\
    &\cdots\\
    &\leq \frac{2\alpha L}{n}(1+\alpha\beta)^{T-2-t_0} + \alpha\beta(1+\alpha\beta)^{T-2-t_0}\Vert w_{t_{0}} - w^{\prime}_{t_{0}} \Vert =\frac{2\alpha L}{n}(1+\alpha\beta)^{T-2-t_0},
  \end{aligned}
 \end{equation}
where $w_{t_{0}} = w^{\prime}_{t_{0}}$. Therefore, we can obtain the bound for $\frac{1}{k}\sum_{i=T-k+1}^{T-1} \alpha \Vert\nabla F(w^{\prime}_{i-1},z_i) - \nabla F(w_{i-1},z_i) \Vert$ in the expectation sense.
\begin{equation}\label{44}
    \begin{aligned}
     \frac{\alpha}{k}\sum_{i=T-k+1}^{T-1} & \mathbb{E} \Vert\nabla F(w^{\prime}_{i-1},z_i) - \nabla F(w_{i-1},z_i) \Vert 
     \leq \frac{2\alpha L}{nk} \sum_{i=T-k}^{T-2} (1+\alpha\beta)^{i-t_{0}} \\
     &\leq \frac{2\alpha L}{nk} \cdot k (1+\alpha\beta)^{T} \leq \frac{2\alpha L(1+\alpha\beta)^{T}}{n}.  
    \end{aligned}
\end{equation}
Then, we obtain the expectation considering the above analysis 
 \begin{equation}
    \begin{aligned}
     \mathbb{E}\left[\bar{\delta}_{T+1}\right] &\leq (1-\frac{1}{n})\left(1+\frac{\alpha \beta(1+\alpha \beta)^{k-1}}{k}\right)\bar{\delta}_T + \frac{1}{n}\left(\bar{\delta}_T+\frac{2\alpha L}{k}\right) + \frac{2\alpha L(1+\alpha\beta)^{T}}{n}\\ 
     &\leq \left(\frac{1}{n}+(1-\frac{1}{n})\left(1+\frac{\alpha \beta(1+\alpha \beta)^{k-1}}{k}\right)\right)\bar{\delta}_{T} + \frac{2\alpha L}{nk}\left(1+k(1+\alpha\beta)^{T}\right)\\
    \end{aligned}
   \end{equation}
let $\alpha = \frac{c}{t}$, then
  \begin{equation}
 \begin{aligned}
     &= \left(1+(1-\frac{1}{n})\frac{c\beta(1+\frac{c\beta}{t})^{k}}{kt}\right) \bar{\delta}_{t} + \frac{2cL}{nkt}\left(1+k(1+\frac{c\beta}{t})^{t}\right)\\
     &\leq \exp\left((1-\frac{1}{n})\frac{c\beta e^{\frac{c\beta k}{t}}}{kt}\right) \bar{\delta}_{t} + \frac{2cL}{kn}\cdot\frac{1+k e^{c\beta}}{t}.
    \end{aligned}
   \end{equation}
Here we used that $\lim\limits_{x\to\infty}(1+\frac{1}{x})^x=e$ and $\lim\limits_{x\to\infty}e^\frac{1}{x}=1$. 
Using the fact that $\bar{\delta}_{t_0}=0$, we can unwind this recurrence relation from $T$ down to $t_0+1$.
  \begin{equation}
    \begin{aligned}
     \mathbb{E}\bar{\delta}_{t} &\leq \sum_{t=t_0 +1}^{T} \left( \prod_{m=t+1}^{T}\exp\left((1-\frac{1}{n})\frac{c\beta}{km}\right)\right)\frac{2cL}{kn}\cdot\frac{1+k e^{c\beta}}{t}\\
     &= \sum_{t=t_0 +1}^{T} \exp\left(\frac{(1-\frac{1}{n})c\beta}{k} \sum_{m=t+1}^{T}\frac{1}{m}\right)\frac{2cL}{kn}\cdot\frac{1+k e^{c\beta}}{t}\\
     &\leq \sum_{t=t_0 +1}^{T} \exp\left( \frac{(1-\frac{1}{n})c\beta}{k} \cdot \log(\frac{T}{t}) \right)\frac{2cL}{kn}\cdot\frac{1+k e^{c\beta}}{t}\\
     &\leq T^{\frac{(1-\frac{1}{n})c\beta}{k}} \cdot \sum_{t=t_0 +1}^{T} \left(\frac{1}{t}\right)^{\frac{(1-\frac{1}{n})c\beta}{k}+1} \cdot \frac{2cL(1+ke^{c\beta})}{kn}\\
     &\leq \frac{k}{(1-\frac{1}{n})c\beta} \cdot \frac{2cL(1+ke^{c\beta})}{kn} \cdot \left(\frac{T}{t_0}\right)^{\frac{(1-\frac{1}{n})c\beta}{k}}\\
     &\leq \frac{2L(1+ke^{c\beta})}{(n-1)\beta} \cdot \left(\frac{T}{t_0}\right)^{\frac{c\beta}{k}}.
    \end{aligned}
   \end{equation}
Plugging this back into Eq.~\eqref{nonconvex-basic}, we obtain
 \begin{equation}\label{with-con}
  \mathbb{E}\vert F(\bar{w}_T^{K};z)-F(\bar{w}^{K\prime}_T;z)\vert \leq \frac{t_0}{n} + \frac{2sL^2(1+ke^{c\beta})}{(n-1)\beta} \cdot \left(\frac{T}{t_0}\right)^{\frac{c\beta}{k}}.
 \end{equation}
By taking the extremum, we obtain the minimum  
 \begin{equation}\label{with-con-t_0}
    t_0 = \left(2csL^2(1+ke^{c\beta})k^{-1}\right)^{\frac{k}{c\beta+k}}\cdot T^{\frac{c\beta}{c\beta+k}}
   \end{equation}
finally, this setting gets
 \begin{equation}\label{with-con-result}
  \epsilon_{gen} = \mathbb{E}\vert F(\bar{w}_T^{K};z)-F(\bar{w}^{K\prime}_T;z)\vert \leq \frac{1+\frac{1}{c\beta}}{n-1}\left(2csL^2(1+ke^{c\beta})k^{-1}\right)^{\frac{k}{c\beta+k}}\cdot T^{\frac{c\beta}{c\beta+k}},
 \end{equation}
to simplify, omitting constant factors that depend on $\beta$, c
and L, we get 
   \begin{equation}
    \epsilon_{gen}  \leq \mathcal{O}_s\left(\frac{T^{\frac{c\beta}{c\beta+k}}}{n}\right).
   \end{equation}
And we finish the proof.

%% file: main_paper.bbl
\begin{thebibliography}{38}
\providecommand{\natexlab}[1]{#1}
\providecommand{\url}[1]{\texttt{#1}}
\expandafter\ifx\csname urlstyle\endcsname\relax
  \providecommand{\doi}[1]{doi: #1}\else
  \providecommand{\doi}{doi: \begingroup \urlstyle{rm}\Url}\fi

\bibitem[Bousquet \& Elisseeff(2002)Bousquet and Elisseeff]{bousquet2002stability}
Bousquet, O. and Elisseeff, A.
\newblock Stability and generalization.
\newblock \emph{The Journal of Machine Learning Research}, 2:\penalty0 499--526, 2002.

\bibitem[Cha et~al.(2021)Cha, Chun, Lee, Cho, Park, Lee, and Park]{cha2021swad}
Cha, J., Chun, S., Lee, K., Cho, H.-C., Park, S., Lee, Y., and Park, S.
\newblock Swad: Domain generalization by seeking flat minima.
\newblock \emph{Advances in Neural Information Processing Systems}, 34:\penalty0 22405--22418, 2021.

\bibitem[Charles \& Papailiopoulos(2018)Charles and Papailiopoulos]{charles2018stability}
Charles, Z. and Papailiopoulos, D.
\newblock Stability and generalization of learning algorithms that converge to global optima.
\newblock In \emph{International conference on machine learning}, pp.\  745--754. PMLR, 2018.

\bibitem[Devroye \& Wagner(1979)Devroye and Wagner]{devroye1979distribution}
Devroye, L. and Wagner, T.
\newblock Distribution-free performance bounds for potential function rules.
\newblock \emph{IEEE Transactions on Information Theory}, 25\penalty0 (5):\penalty0 601--604, 1979.

\bibitem[Fu et~al.(2020)Fu, Kumar, Nachum, Tucker, and Levine]{fu2020d4rl}
Fu, J., Kumar, A., Nachum, O., Tucker, G., and Levine, S.
\newblock D4rl: Datasets for deep data-driven reinforcement learning.
\newblock \emph{arXiv preprint arXiv:2004.07219}, 2020.

\bibitem[Hardt et~al.(2016)Hardt, Recht, and Singer]{hardt2016train}
Hardt, M., Recht, B., and Singer, Y.
\newblock Train faster, generalize better: Stability of stochastic gradient descent.
\newblock In \emph{International conference on machine learning}, pp.\  1225--1234. PMLR, 2016.

\bibitem[Huang et~al.(2020)Huang, Huang, Li, and Li]{huang2020meta}
Huang, Y., Huang, W., Li, L., and Li, Z.
\newblock Meta-learning pac-bayes priors in model averaging.
\newblock In \emph{Proceedings of the AAAI Conference on Artificial Intelligence}, volume~34, pp.\  4198--4205, 2020.

\bibitem[Izmailov et~al.(2018)Izmailov, Podoprikhin, Garipov, Vetrov, and Wilson]{izmailov2018averaging}
Izmailov, P., Podoprikhin, D., Garipov, T., Vetrov, D., and Wilson, A.~G.
\newblock Averaging weights leads to wider optima and better generalization.
\newblock \emph{arXiv preprint arXiv:1803.05407}, 2018.

\bibitem[Jang et~al.(2017)Jang, Gu, and Poole]{jang2017categorical}
Jang, E., Gu, S., and Poole, B.
\newblock Categorical reparameterization with gumbel-softmax.
\newblock In \emph{International Conference on Learning Representations}, 2017.
\newblock URL \url{https://openreview.net/forum?id=rkE3y85ee}.

\bibitem[Kaddour(2022)]{kaddour2022stop}
Kaddour, J.
\newblock Stop wasting my time! saving days of imagenet and bert training with latest weight averaging.
\newblock \emph{arXiv preprint arXiv:2209.14981}, 2022.

\bibitem[Kingma \& Welling(2013)Kingma and Welling]{kingma2013auto}
Kingma, D.~P. and Welling, M.
\newblock Auto-encoding variational bayes.
\newblock \emph{arXiv preprint arXiv:1312.6114}, 2013.

\bibitem[Kuzborskij \& Lampert(2018)Kuzborskij and Lampert]{kuzborskij2018data}
Kuzborskij, I. and Lampert, C.
\newblock Data-dependent stability of stochastic gradient descent.
\newblock In \emph{International Conference on Machine Learning}, pp.\  2815--2824. PMLR, 2018.

\bibitem[Lei \& Ying(2020)Lei and Ying]{lei2020sharper}
Lei, Y. and Ying, Y.
\newblock Sharper generalization bounds for learning with gradient-dominated objective functions.
\newblock In \emph{International Conference on Learning Representations}, 2020.

\bibitem[Li et~al.(2022)Li, Huang, Tao, Wu, and Huang]{li2022trainable}
Li, T., Huang, Z., Tao, Q., Wu, Y., and Huang, X.
\newblock Trainable weight averaging: Efficient training by optimizing historical solutions.
\newblock In \emph{The Eleventh International Conference on Learning Representations}, 2022.

\bibitem[Li et~al.(2023)Li, Peng, Zhang, Ding, Hu, and Shen]{li2023deep}
Li, W., Peng, Y., Zhang, M., Ding, L., Hu, H., and Shen, L.
\newblock Deep model fusion: A survey.
\newblock \emph{arXiv preprint arXiv:2309.15698}, 2023.

\bibitem[Lu et~al.(2022)Lu, Kobyzev, Rezagholizadeh, Rashid, Ghodsi, and Langlais]{lu2022improving}
Lu, P., Kobyzev, I., Rezagholizadeh, M., Rashid, A., Ghodsi, A., and Langlais, P.
\newblock Improving generalization of pre-trained language models via stochastic weight averaging.
\newblock In \emph{Findings of the Association for Computational Linguistics: EMNLP 2022}, pp.\  4948--4954, 2022.

\bibitem[Maddison et~al.(2017)Maddison, Mnih, and Teh]{maddison2017the}
Maddison, C.~J., Mnih, A., and Teh, Y.~W.
\newblock The concrete distribution: A continuous relaxation of discrete random variables.
\newblock In \emph{International Conference on Learning Representations}, 2017.
\newblock URL \url{https://openreview.net/forum?id=S1jE5L5gl}.

\bibitem[Mukherjee et~al.(2006)Mukherjee, Niyogi, Poggio, and Rifkin]{mukherjee2006learning}
Mukherjee, S., Niyogi, P., Poggio, T., and Rifkin, R.
\newblock Learning theory: stability is sufficient for generalization and necessary and sufficient for consistency of empirical risk minimization.
\newblock \emph{Advances in Computational Mathematics}, 25:\penalty0 161--193, 2006.

\bibitem[Polyak \& Juditsky(1992)Polyak and Juditsky]{polyak1992acceleration}
Polyak, B.~T. and Juditsky, A.~B.
\newblock Acceleration of stochastic approximation by averaging.
\newblock \emph{SIAM journal on control and optimization}, 30\penalty0 (4):\penalty0 838--855, 1992.

\bibitem[Rezende et~al.(2014)Rezende, Mohamed, and Wierstra]{rezende2014stochastic}
Rezende, D.~J., Mohamed, S., and Wierstra, D.
\newblock Stochastic backpropagation and variational inference in deep latent gaussian models.
\newblock In \emph{International conference on machine learning}, volume~2, pp.\ ~2, 2014.

\bibitem[Ruppert(1988)]{ruppert1988efficient}
Ruppert, D.
\newblock Efficient estimations from a slowly convergent robbins-monro process.
\newblock Technical report, Cornell University Operations Research and Industrial Engineering, 1988.

\bibitem[Sanyal et~al.(2023)Sanyal, Neerkaje, Kaddour, Kumar, et~al.]{sanyal2023early}
Sanyal, S., Neerkaje, A.~T., Kaddour, J., Kumar, A., et~al.
\newblock Early weight averaging meets high learning rates for llm pre-training.
\newblock In \emph{Workshop on Advancing Neural Network Training: Computational Efficiency, Scalability, and Resource Optimization (WANT@ NeurIPS 2023)}, 2023.

\bibitem[Shalev-Shwartz et~al.(2010)Shalev-Shwartz, Shamir, Srebro, and Sridharan]{shalev2010learnability}
Shalev-Shwartz, S., Shamir, O., Srebro, N., and Sridharan, K.
\newblock Learnability, stability and uniform convergence.
\newblock \emph{The Journal of Machine Learning Research}, 11:\penalty0 2635--2670, 2010.

\bibitem[Shamir(2016)]{shamir2016without}
Shamir, O.
\newblock Without-replacement sampling for stochastic gradient methods.
\newblock \emph{Advances in neural information processing systems}, 29, 2016.

\bibitem[Sun et~al.(2023{\natexlab{a}})Sun, Shen, and Tao]{sun2023mode}
Sun, Y., Shen, L., and Tao, D.
\newblock Which mode is better for federated learning? centralized or decentralized.
\newblock \emph{arXiv preprint arXiv:2310.03461}, 2023{\natexlab{a}}.

\bibitem[Sun et~al.(2023{\natexlab{b}})Sun, Shen, and Tao]{sun2023understanding}
Sun, Y., Shen, L., and Tao, D.
\newblock Understanding how consistency works in federated learning via stage-wise relaxed initialization.
\newblock \emph{arXiv preprint arXiv:2306.05706}, 2023{\natexlab{b}}.

\bibitem[Sutton et~al.(1999)Sutton, McAllester, Singh, and Mansour]{sutton1999policy}
Sutton, R.~S., McAllester, D., Singh, S., and Mansour, Y.
\newblock Policy gradient methods for reinforcement learning with function approximation.
\newblock \emph{Advances in neural information processing systems}, 12, 1999.

\bibitem[Szegedy et~al.(2016)Szegedy, Vanhoucke, Ioffe, Shlens, and Wojna]{szegedy2016rethinking}
Szegedy, C., Vanhoucke, V., Ioffe, S., Shlens, J., and Wojna, Z.
\newblock Rethinking the inception architecture for computer vision.
\newblock In \emph{Proceedings of the IEEE conference on computer vision and pattern recognition}, pp.\  2818--2826, 2016.

\bibitem[Wang et~al.(2024{\natexlab{a}})Wang, Shen, Tao, He, and Tao]{wanggeneralization}
Wang, P., Shen, L., Tao, Z., He, S., and Tao, D.
\newblock Generalization analysis of stochastic weight averaging with general sampling.
\newblock In \emph{Forty-first International Conference on Machine Learning}, 2024{\natexlab{a}}.

\bibitem[Wang et~al.(2024{\natexlab{b}})Wang, Shen, Tao, Sun, and Tao]{peng2020dfwa}
Wang, P., Shen, L., Tao, Z., Sun, Yan~Zheng, G., and Tao, D.
\newblock A unified analysis for finite weight averaging.
\newblock \emph{arXiv preprint arXiv:2411.13169}, 2024{\natexlab{b}}.

\bibitem[Williams(1992)]{williams1992simple}
Williams, R.~J.
\newblock Simple statistical gradient-following algorithms for connectionist reinforcement learning.
\newblock \emph{Machine learning}, 8:\penalty0 229--256, 1992.

\bibitem[Xiao et~al.(2022)Xiao, Fan, Sun, Wang, and Luo]{xiao2022stability}
Xiao, J., Fan, Y., Sun, R., Wang, J., and Luo, Z.-Q.
\newblock Stability analysis and generalization bounds of adversarial training.
\newblock \emph{Advances in Neural Information Processing Systems}, 35:\penalty0 15446--15459, 2022.

\bibitem[Yang et~al.(2021)Yang, Lei, Wang, Yang, and Ying]{yang2021simple}
Yang, Z., Lei, Y., Wang, P., Yang, T., and Ying, Y.
\newblock Simple stochastic and online gradient descent algorithms for pairwise learning.
\newblock \emph{Advances in Neural Information Processing Systems}, 34:\penalty0 20160--20171, 2021.

\bibitem[Yuan et~al.(2019)Yuan, Yan, Jin, and Yang]{yuan2019stagewise}
Yuan, Z., Yan, Y., Jin, R., and Yang, T.
\newblock Stagewise training accelerates convergence of testing error over sgd.
\newblock \emph{Advances in Neural Information Processing Systems}, 32, 2019.

\bibitem[Zhang et~al.(2024)Zhang, Zhang, Pi, Jin, Gao, Ye, and Chen]{zhangefficient}
Zhang, W., Zhang, Z., Pi, R., Jin, Z., Gao, Y., Ye, J., and Chen, K.
\newblock Efficient denoising diffusion via probabilistic masking.
\newblock In \emph{Forty-first International Conference on Machine Learning}, 2024.

\bibitem[Zhou et~al.(2022)Zhou, Pi, Zhang, Lin, and Zhang]{coreset}
Zhou, X., Pi, R., Zhang, W., Lin, Y., and Zhang, T.
\newblock Probabilistic bilevel coreset selection.
\newblock In \emph{International Conference on Machine Learning}. PMLR, 2022.

\bibitem[Zhou et~al.(2018)Zhou, Liang, and Zhang]{zhou2018generalization}
Zhou, Y., Liang, Y., and Zhang, H.
\newblock Generalization error bounds with probabilistic guarantee for sgd in nonconvex optimization.
\newblock \emph{arXiv preprint arXiv:1802.06903}, 2018.

\bibitem[Zhu et~al.(2023)Zhu, Shen, Du, and Tao]{zhu2023stability}
Zhu, M., Shen, L., Du, B., and Tao, D.
\newblock Stability and generalization of the decentralized stochastic gradient descent ascent algorithm.
\newblock In \emph{Thirty-seventh Conference on Neural Information Processing Systems}, 2023.

\end{thebibliography}
